\documentclass{article}

\usepackage{microtype}
\usepackage{graphicx}
\usepackage{subcaption}
\usepackage{booktabs} 
\usepackage{multirow}
\usepackage{hyperref}

\usepackage[accepted]{icml2026}

\usepackage{amsmath}
\usepackage{amssymb}
\usepackage{mathtools}
\usepackage{amsthm}

\usepackage[capitalize,noabbrev]{cleveref}

\theoremstyle{plain}
\newtheorem{theorem}{Theorem}[section]
\newtheorem{proposition}[theorem]{Proposition}
\newtheorem{lemma}[theorem]{Lemma}

\theoremstyle{definition}

\theoremstyle{remark}

\usepackage[textsize=tiny]{todonotes}

\icmltitlerunning{Training-Free Coverless Multi-Image Steganography with Access Control}

\begin{document}

\twocolumn[
  \icmltitle{Training-Free Coverless Multi-Image Steganography with Access Control}



  \icmlsetsymbol{equal}{*}

  \begin{icmlauthorlist}
    \icmlauthor{Minyeol Bae}{yyy}
    \icmlauthor{Si-Hyeon Lee}{yyy}
  \end{icmlauthorlist}

  \icmlaffiliation{yyy}{School of Electrical Engineering, Korea Advanced Institute of
Science and Technology (KAIST), Daejeon 34141, South Korea}

  \icmlcorrespondingauthor{Si-Hyeon Lee}{sihyeon@kaist.ac.kr}

  \icmlkeywords{Image Steganography, Coverless Steganography, Diffusion Models, Access Control, Training-Free}

  \vskip 0.3in
]



\printAffiliationsAndNotice{}  

\begin{abstract}
Coverless Image Steganography (CIS) hides information without explicitly modifying a cover image, providing strong imperceptibility and inherent robustness to steganalysis. However, existing CIS methods largely lack robust access control, making it difficult to selectively reveal different hidden contents to different authorized users. Such access control is critical for scalable and privacy-sensitive information hiding in multi-user settings. We propose MIDAS (\textbf{M}ulti-\textbf{I}mage \textbf{D}iffusion-based \textbf{A}ccess-controlled \textbf{S}teganography), a training-free diffusion-based CIS framework that enables multi-image hiding with user-specific access control via latent-level fusion. {MIDAS introduces a Random Basis mechanism to suppress residual structural information, together with a theoretical analysis of information leakage, and a Latent Vector Fusion module that reshapes aggregated latents to better align with the diffusion process.}  Experimental results demonstrate that MIDAS consistently outperforms existing training-free CIS baselines in access control functionality, stego image quality and diversity, robustness to noise, and resistance to steganalysis, establishing a practical and scalable approach to access-controlled coverless steganography.
\end{abstract}

\section{Introduction}
\label{sec:intro}
In the era of pervasive internet connectivity and AI-generated content (AIGC), data privacy and security have become increasingly important concerns \cite{CRoSS}. Image steganography is a widely used technique for secure communication, where secret information is embedded into visually natural images in a perceptually imperceptible manner, such that only authorized users possessing the secret key or extraction mechanism can recover the hidden content \cite{prisoners}. Image steganography has been applied to a wide range of scenarios, including social media platforms \cite{social_stega}, edge computing \cite{edge_computing_stega}, quantum computing \cite{quantum_stega}, and efficient image compression \cite{stego_compress}.
\begin{figure}[t]
\begin{center}
\centerline{\includegraphics[width=\columnwidth]{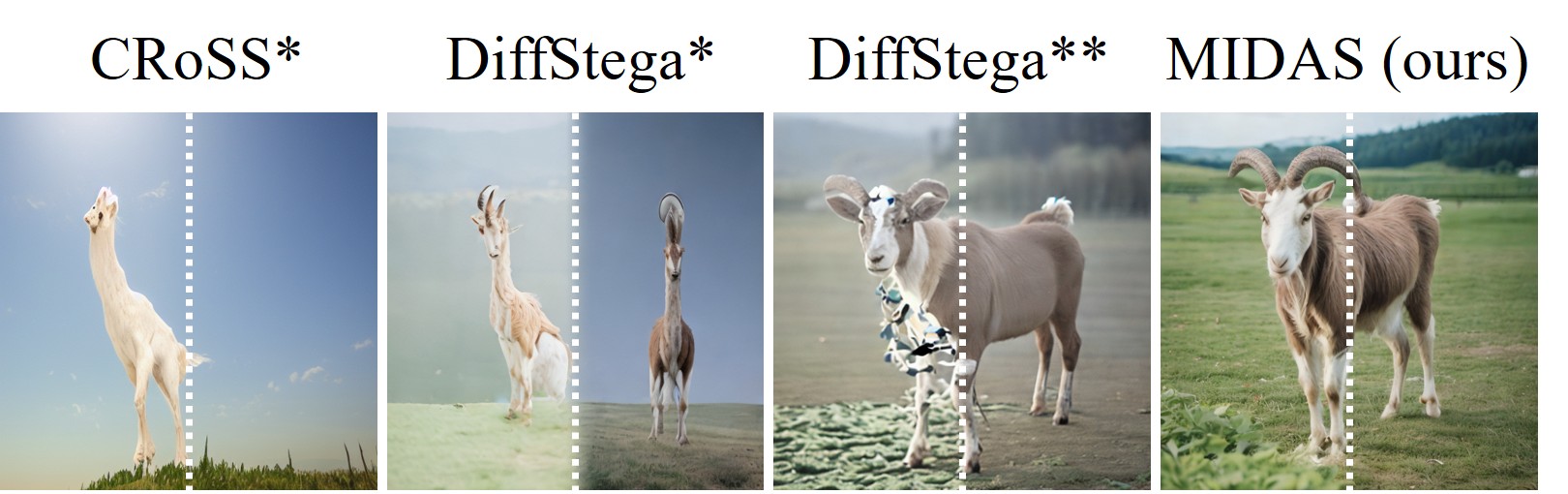}}
\caption{Comparison of stego image quality across training-free CIS methods {with two hidden secret images}.
{CRoSS*, DiffStega*, and DiffStega** are extensions of CRoSS and DiffStega based on latent concatenation} (see Section~\ref{sec:exp setting} and~Appendix \ref{appendix:impl_detail} for details). Dashed lines indicate concatenation boundaries. While baseline methods often suffer from visual artifacts or discontinuities due to limited diversity, MIDAS preserves high visual fidelity and natural image synthesis.}
\label{comparison}
\end{center}
\vskip -0.4in
\end{figure}


Traditional image steganography assumes a natural image, referred to as a cover image, and focuses on embedding bit-level messages by modifying the cover at the pixel or transform level \cite{spatial_LSB, DCT, JPEG_DCT}, yielding a modified image referred to as a stego image. These approaches are collectively known as modification-based methods. More recently, this modification-based paradigm has been extended to scenarios where secret images are embedded into a cover image \cite{Baluja}. To improve both reconstruction quality and embedding capacity, a growing body of work has leveraged advanced deep neural architectures \cite{HiNet, ISN, DeepMIH, RIIS}.


\setlength{\tabcolsep}{2pt}
\begin{table}[t]
\caption{Comparison of image steganography methods.}
\label{tab:baselines}
\begin{center}
{
\fontsize{8pt}{10pt}\selectfont
\begin{sc}
\begin{tabular}{ccccc}
\toprule
\normalfont Method & \normalfont Coverless & \normalfont Training & \normalfont Multi-image & \normalfont Access control \\
\midrule
\normalfont Baluja & & \normalfont moderate & \checkmark &  \\
\normalfont HiNet & & \normalfont moderate & & \\
\normalfont ISN & & \normalfont moderate & \checkmark & \\
\normalfont DeepMIH & & \normalfont moderate & \checkmark & \\
\normalfont RIIS & & \normalfont moderate & \checkmark & \\
\midrule
\normalfont Kweon et al. & & \normalfont moderate & \checkmark & \checkmark \\
\normalfont IIS & & \normalfont moderate & \checkmark & \checkmark \\
\normalfont AIS & & \normalfont moderate & \checkmark & \checkmark \\
\midrule
\normalfont CRoSS & \checkmark & \normalfont free & & \\
\normalfont DiffStega & \checkmark & \normalfont free & & \\
\normalfont DStyleStego & \checkmark & \normalfont free & & \\
\midrule
\normalfont HIS & \checkmark & \normalfont moderate & \checkmark & \\
\normalfont Qin et al. & \checkmark & \normalfont heavy & \checkmark & \\
\normalfont Chen et al. & \checkmark & \normalfont heavy & & \\
\midrule
\normalfont MIDAS(Ours) & \checkmark & \normalfont free & \checkmark & \checkmark \\
\bottomrule
\end{tabular}
\end{sc}
}
\end{center}
\vskip -0.2in
\end{table}
In multi-image hiding scenarios, multiple secret images are embedded into a single cover image to improve transmission efficiency. In practice, each secret image may be intended for a different receiver, which necessitates access control functionality where only authorized users can recover their designated secrets. Such access control is indispensable in the private sector for managing sensitive databases and protecting industrial secrets, as well as in organizations with differentiated security requirements, including government, military, and diplomatic sectors \cite{Access_control}. Recent works have explored modification-based approaches to enable access control in multi-image steganography \cite{Kweon, IIS, AIS}. However, it is well known that such modification-based methods render the resulting stego image vulnerable to steganalysis, particularly when the original cover image is leaked or predictable, thereby undermining the fundamental security of the hidden data.

Meanwhile, coverless image steganography (CIS) has emerged as an alternative paradigm that avoids modifying a cover image, often leveraging generative models, thereby achieving strong robustness against steganalysis. Representative CIS methods, such as CRoSS \cite{CRoSS} and DiffStega \cite{DiffStega}, are particularly appealing because they are training-free, which is crucial for real-world applications given the high computational cost and data requirements associated with training generative models. Furthermore, training-based schemes typically require distributing the entire trained model, which effectively acts as an impractically long secret key. In contrast, training-free schemes rely on publicly available models, enabling significantly more practical and efficient deployment. However, existing training-free CIS methods still suffer from two key limitations:

(1) \textbf{Absence of Multi-Image Hiding and Access Control}: 
Access control functionality inherently requires reliable and effective multi-image hiding; however, existing training-free CIS methods are primarily designed for single-image concealment, and naive extensions to multi-image hiding lead to severe performance degradation.

(2) \textbf{Limited Stego Image Diversity}: 
The generated stego images often preserve substantial residual structural information \cite{Concat_Mapping}. This limitation is particularly pronounced in naive concatenation-based extensions for access control scenarios (see Fig.~\ref{comparison}).



Several recent works \cite{Concat_Mapping, Qin, HIS, DStyleStego} have attempted to address these issues; however, they either require training or rely on the transmission of side information associated with the secret image (see Section \ref{sec:CIS} for details). As a result, a training-free CIS framework with access control capability that does not require transmitting any additional secret-related information is still lacking.


\begin{figure*}[ht]
\begin{center}
\center{\includegraphics[width=6.5in]{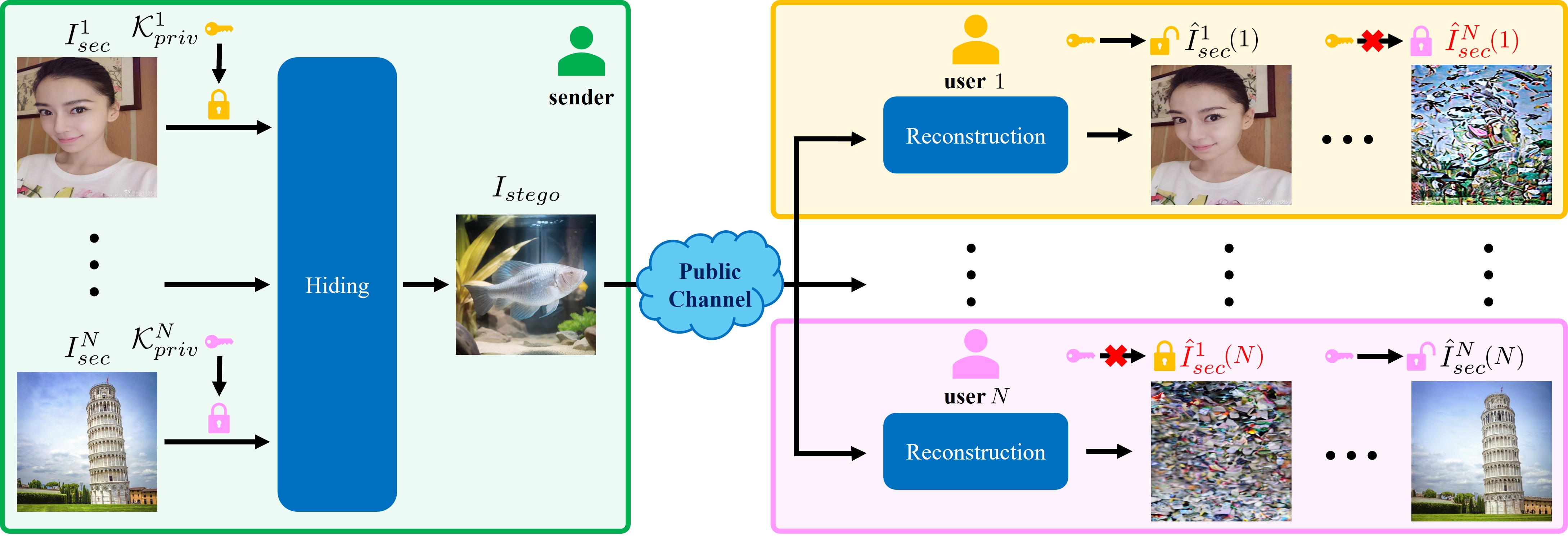}}
\caption{The overall scenario. The sender embeds multiple secret images ($I_{sec}^1, ..., I_{sec}^N$) into a single stego image ($I_{stego}$), which is then transmitted through a public channel. Each secret image ${I}_{sec}^i$ is encrypted with a private key ${\mathcal{K}}_{priv}^i$. Let $\hat{I}_{sec}^{j}(i)$ denote the reconstruction result of secret image~$j$ obtained by user~$i$ using $\mathcal{K}_{priv}^{i}$.
 Due to the access control function, only the intended receiver $i$ possessing ${\mathcal{K}}_{priv}^i$ can successfully recover $\hat{I}_{sec}^i(i)$, whereas for any other image $j \neq i$, the reconstruction $\hat{I}_{sec}^j(i)$ results in a meaningless output.}

\label{scenario}
\end{center}
\vskip -0.2in
\end{figure*}



To address this gap, we introduce a novel training-free CIS framework that supports access-controlled multi-image hiding without transmitting any auxiliary information related to the secret images. The main contributions of this work are summarized as follows:

\begin{itemize}
    \item 
  We propose \textbf{MIDAS} (\textbf{M}ulti-\textbf{I}mage \textbf{D}iffusion-based \textbf{A}ccess-controlled \textbf{S}teganography), a novel training-free CIS framework {for access-controlled multi-image hiding that leverages pre-trained diffusion models and seamlessly integrates into diffusion pipelines without modifying the underlying model parameters.}
    
    \item 
    We introduce a Latent Vector Fusion module and a key mechanism termed Random Basis. The Random Basis mechanism neutralizes the inherent structure of noisy latent representations {and is accompanied by an information-theoretic analysis of information leakage, while} the Latent Vector Fusion module reshapes them into representations that are more amenable to the diffusion process. Together, these components significantly improve both the perceptual quality and diversity of the generated stego images (see Fig.~\ref{comparison}). 
    
    \item  
 Through extensive experiments, we demonstrate that MIDAS consistently outperforms existing training-free CIS baselines in terms of access control functionality, stego image quality and diversity, robustness to noise, and resistance to steganalysis.
\end{itemize}

{Table~\ref{tab:baselines} summarizes the distinguishing features of our work in comparison with existing studies, and Section~\ref{sec:related work} presents a more detailed review of the related literature.}\footnote{The source code of the proposed algorithm is available at https://github.com/Minyeol/MIDAS.}

\paragraph{Conflict of Interest Disclosure.} The authors have no conflicts of interest to declare.

\section{Related Work and System Model}
\label{sec:system model and related work}
\subsection{Related Work}
\label{sec:related work}
\subsubsection{{Cover-Based Image} Steganography} 
\label{sec:image stega}

{The paradigm of} hiding secret images within other images has gained widespread interest, a concept first introduced by Baluja \yrcite{Baluja} using deep neural networks. Building on this foundation, several studies have adopted Invertible Neural Networks (INNs) to improve reconstruction quality, including HiNet \cite{HiNet}, ISN \cite{ISN}, DeepMIH \cite{DeepMIH}, and RIIS \cite{RIIS}. 

Despite these advancements, {achieving access-controlled multi-image hiding remains a significant challenge; incorporating} user-selective {private} keys often requires additional metadata that can degrade the quality of both stego and reconstructed images \cite{AIS}.
{Kweon et al. \yrcite{Kweon} integrate {the private} keys into multi-image steganography, providing a foundation for access control. Following this, IIS \cite{IIS} and AIS \cite{AIS} employ INN-based architectures to further enhance reconstruction quality and capacity. Nevertheless, these methods are modification-based and thus inherently vulnerable to steganalysis. In contrast, MIDAS achieves access control} in a coverless manner, thereby fundamentally enhancing security by eliminating the need for cover image modification.

\subsubsection{Coverless Image Steganography}
\label{sec:CIS}
CIS is fundamentally robust against steganalysis as it avoids cover image modification. 
Many CIS schemes utilize generative models to synthesize stego images directly from secret images. However, optimizing high-quality generators is a resource-intensive process. In this regard, training-free approaches offer significant advantages. CRoSS \cite{CRoSS} utilizes the {Denoising Diffusion Implicit Model (DDIM) \cite{DDIM} inversion process {to hide data in two stages: a secret image is first mapped to a noisy latent using a private prompt (private key), then transformed into a stego image via the backward process guided by a public prompt (public key). The secret image is subsequently retrieved by exploiting the inherent reversibility of DDIM.} 
Note that CRoSS faces notable challenges: it relies on secret image-dependent private prompts that must be transmitted for each session, and its limited editing capabilities often leave residual structural information in the stego image, thereby hindering naive extensions to multi-image hiding.
DiffStega \cite{DiffStega} eliminates the need for private prompts by replacing them with a fixed null-text prompt and applying a private key mechanism directly to noisy latents. While it improves stego image quality by integrating a reference image and ControlNet \cite{controlnet}, its editing capability remains restricted. DStyleStego \cite{DStyleStego} improves editing performance, but it requires the transmission of information related to the secret image and remains restricted to single-image hiding. Consequently, these training-free methods are either not readily extendable to multi-image hiding or require the transmission of secret-related information.


Some recent CIS methods have attempted to overcome  the aforementioned limitations through specialized training. For instance, Chen et al. \yrcite{Concat_Mapping} and Qin et al. \yrcite{Qin} {employ steganography-specific generators to enhance editing capabilities, with the latter further enabling multi-image hiding scenarios.} However, these methods entail substantial computational overhead due to their reliance on high-quality generator training.
Another approach, HIS \cite{HIS}, generates a stego image via CRoSS to serve as a cover  for subsequent modification-based steganography {to enable multi-image hiding}, while still necessitating the transmission of {additional information related to secret images} for every communication. Furthermore, none of the aforementioned approaches provide an explicit access control function. In contrast, MIDAS enables access-controlled {CIS} even within a training-free framework {without requiring the transmission of secret-related information, thereby} offering a practical and secure solution for real-world scenarios.

\subsection{System Model}
\label{sec:system model}

We consider a multi-image access-controlled steganography scenario with a single sender and $N$ receivers (users), as illustrated in Fig.~\ref{scenario}. The sender possesses $N$ secret images $(I_{sec}^1, ... , I_{sec}^N)$ and $N$ private keys $(\mathcal K_{priv}^1, ..., \mathcal K_{priv}^N)$, and the $i$-th receiver has the private key $\mathcal K_{priv}^i$. We further allow for the existence of certain public resources shared between the sender and the receivers. Under a strong attacker model, these public resources are assumed to be fully accessible to the attacker as well. 
All the secret images are jointly embedded into a single stego image $I_{stego}$ using the private keys and the public resources, which is transmitted over a public channel. The resulting stego image is required to be visually indistinguishable from a natural image.
Upon receiving $I_{{stego}}$, each receiver attempts to reconstruct the embedded secret images. The $i$-th receiver, who possesses only the private key $\mathcal{K}_{priv}^i$, is required to obtain a faithful reconstruction $\hat I_{sec}^i(i)$ of its designated secret image $I_{{sec}}^i$. In contrast, for any other image $j \neq i$, the reconstructed output $\hat I_{sec}^j(i)$ by the $i$-th receiver should reveal no useful information about $I_{sec}^j$. {Accordingly, receivers without the corresponding private key are also treated as attackers in our access-control model, and should gain no useful information about unauthorized secret images.}


\section{Method}
\label{sec:method}
\subsection{Overview}
\label{sec:overview}

This section introduces MIDAS, a novel access-controlled image steganography method utilizing a diffusion model. 
For both Hiding and Reconstruction Stages, MIDAS employs a latent diffusion model \cite{LDM} with a DDIM sampler \cite{DDIM}. The inherent inversion property of DDIM is utilized to enable the reconstruction of secret images. In addition, it integrates EDICT \cite{EDICT}, an exact inversion method, to further improve {the reconstruction quality}. In MIDAS, private keys are given as random seeds. The public resources consist of a public key $\mathcal K_{pub}$ corresponding to an additional random seed and a public prompt $\mathcal{P}_{pub}$, which together facilitate the generation of stego images. The model operates in a latent space with $C \times H \times W$ dimensions, where $C$ denotes the channel dimension and $H$ and $W$ are the spatial dimensions of the latent representation. The Hiding and Reconstruction Stages of MIDAS are illustrated in Fig.~\ref{architecture} of Appendix~\ref{appendix:architecture}, and are described in Sections~\ref{sec:hiding stage} and~\ref{sec:recon stage}, respectively. 

In MIDAS, a Reference Generator (RefGen), utilizing $\mathcal K_{pub}$ and $\mathcal P_{pub}$, generates a reference image $I_{ref}$. $I_{ref}$ is essential for enhancing the visual quality of $I_{stego}$. Since $I_{ref}$ is generated deterministically from public resources ($\mathcal K_{pub}$ and $\mathcal P_{pub}$), it does not need to be explicitly transmitted. The detailed operation and architecture of RefGen are discussed in Section \ref{sec:RefGen}.

\subsection{Hiding Stage}
\label{sec:hiding stage}
In this stage, $N$ secret images ($I_{sec}^1, ..., I_{sec}^N$) are first individually encrypted using their corresponding private keys ($\mathcal K_{priv}^1, ..., \mathcal K_{priv}^N$). The resulting encrypted data are then jointly processed and fused to a single stego image ($I_{stego}$). This stage consists of the following steps.

\paragraph{Forward Diffusion.} Each secret image $I_{sec}^i$ is first downsampled and mapped to a noisy latent representation $\textbf{z}_{sec}^i\in \mathbb R^{C\times \frac H {N_1}\times \frac W {N_2}}$ using a forward diffusion module, guided by a null-text prompt. The downsampling factors, $N_1$ and $N_2$ (where $N_1N_2 = N$), are strategically chosen to ensure that the concatenated latent vector of all $N$ secret images precisely matches the base latent dimension of the diffusion model, $\mathbb{R}^{C \times H \times W}$, e.g., $N_1 = 2$ and $N_2=1$ when $N=2$, and $N_1=N_2=2$ when $N=4$.

\paragraph{Private Key Encryption.} For user access control, each $\textbf{z}_{sec}^i$ is encrypted into $\textbf{z}_{prot}^i\in \mathbb R^{C\times \frac H {N_1}\times \frac W {N_2}}$ using the private key $\mathcal K_{priv}^i$. For any $d$-dimensional input vector $\textbf{z}$, we consider an encryption strategy that can be expressed in the following form:
\begin{equation}
\label{eq: Random Basis}
    \textbf{z}_{enc} = M_d\textbf{z},    
\end{equation}
where $\textbf{z}_{enc}\in\mathbb R^d$ is the encrypted latent vector, $\textbf{z}\in\mathbb R^d$ is the vector to be encrypted, and $M_d\in\mathbb R^{d\times d}$ is an orthonormal matrix chosen by the encryption strategy.  Due to the orthonormal property of $M_d$, the original vector can be perfectly reconstructed as $\textbf{z}=M_d^T\textbf{z}_{enc}$, where $M_d^T$ denotes the transpose of $M_d$. 
The previously suggested Noise Flip \cite{DiffStega} mechanism, which flips the signs of the elements of the input vector, used $M_d=\text{diag}(e)$ with $e\in\{-1,1\}^d$. 

We {employ a Random Basis mechanism}, which constructs $M_d=Q_d(\mathcal K, \gamma)$, where $Q_d(\mathcal K, \gamma)$ is an orthonormal matrix generated from the seed $\mathcal K$ and the strength $\gamma$. The hyperparameter $\gamma$ determines the proportion of $\mathbf{z}$'s elements that are affected by the transformation $Q_d$, i.e., $\gamma$ portion of $\textbf{z}$'s elements are transformed by $Q_d$ while the remaining elements are left identical. Further details regarding the construction of $Q_d$ can be found in Appendix \ref{appendix:random_basis}. 

By applying the Random Basis to $\textbf{z}_{sec}^i$ with the private key $\mathcal K_{priv}^i$ and strength $\gamma_{priv}$, we have $\textbf{z}_{prot}^i=Q_d(\mathcal K_{priv}^i, \gamma_{priv})\textbf{z}_{sec}^i$. 
{Despite its conceptual simplicity, t}he Random Basis empirically demonstrates superior performance compared to Noise Flip; further details are provided in {Appendix} \ref{appendix:key_mechanism}.

\paragraph{Latent Vector Fusion.} To enable multi-image hiding within a single transmission, the sender concatenates all $\textbf{z}_{prot}^i$ vectors to form a combined latent vector, denoted as $\textbf{z}_{prot}\in \mathbb R^{C\times H\times W}$. However, simple concatenation results in a stego image with abrupt and visible boundaries at the concatenated interfaces (see Fig. \ref{comparison}). This phenomenon occurs because residual information of the original secret images remains within the noisy latent representations ($\textbf{z}_{prot}^i$) even after the DDIM forward process \cite{seed_to_seed}. Consequently, the backward diffusion model struggles to smoothly blend these structurally distinct latent segments. This results in a stego image that is visibly fragmented along the concatenation seam, rendering the output image unnatural and unsuitable for steganography. 

To mitigate this, a Latent Vector Fusion step utilizing the public key $\mathcal K_{pub}$ is applied. This process shuffles the concatenated latent vector $\mathbf{z}_{prot}$ and incorporates some reference component, thereby ensuring the generation of a visually high-quality stego image. The Latent Vector Fusion process is defined as: 
\begin{equation}
\label{eq:LVF}
    \textbf{z}_{pub}=\sqrt{\alpha} M_{D}\textbf{z}_{prot} + \sqrt{1-\alpha} \textbf{z}_{ref}.    
\end{equation}
Here, $\textbf{z}_{pub}\in \mathbb R^{C\times H\times W}$ is the resulting latent vector, $\textbf{z}_{prot}$ is the concatenated vector. 
$M_{D}=Q_{D}(\mathcal K_{pub}, \gamma_{fuse})\in\mathbb R^{D\times D}$ is constructed via the Random Basis mechanism with input dimension $D=C \times H \times W$, public key $\mathcal K_{pub}$, and strength $\gamma_{fuse}$. This transformation serves to mix the spatial components of $\mathbf{z}_{prot}$, thereby removing the localized spatial information caused by the concatenation. $\textbf{z}_{ref}\in\mathbb R^{C\times H\times W}$ is the noisy latent vector of the reference image $I_{ref}$, used to further enhance the stego image generation quality. $\alpha$ is a mixing coefficient that controls the proportions of the transformed vector component ($M_{D}\mathbf{z}_{prot}$) and $\mathbf{z}_{ref}$.

\paragraph{Reverse Diffusion.} The stego image $I_{stego}$ is generated by applying the reverse diffusion process to $\textbf{z}_{pub}$ conditioned on the public prompt $\mathcal P_{pub}$ and the reference image $I_{ref}$. $I_{stego}$ is then transmitted through the public channel. 
\subsection{Reconstruction Stage}
\label{sec:recon stage}
In this stage, each user reconstructs their corresponding secret image. Crucially, this stage verifies the user's identity using their private key, thereby ensuring that only the intended recipient can correctly decode their assigned secret image from $I_{stego}$. Since $I_{stego}$ may be damaged during transmission, each user receives $\tilde I_{stego}$, a potentially degraded version of $I_{stego}$. {This stage} proceeds through the following steps.

\paragraph{Forward Diffusion.} Using DDIM inversion, the received image $\tilde I_{stego}$ is inverted to infer $\tilde{\textbf{z}}_{pub}\in\mathbb R^{C\times H\times W}$, an approximation of $\textbf{z}_{pub}$. $\mathcal P_{pub}$ and $I_{ref}$ are used as conditions for this DDIM inversion.

\paragraph{Latent Vector Decomposition.} In this step, each user performs the inverse operation of the Latent Vector Fusion. From $\tilde{\textbf{z}}_{pub}$, the user first remove the $\textbf{z}_{ref}$ component and then apply the inverse Random Basis transformation to obtain $\hat{\textbf{z}}_{prot}\in\mathbb R^{C\times H\times W}$. This decomposition process is represented by the following equation:
\begin{equation}
 \hat{\textbf{z}}_{prot} = M_D^T \left(\frac{\tilde{\textbf{z}}_{pub} - \sqrt{1-\alpha}\textbf{z}_{ref}}{\sqrt{\alpha}}\right)
\end{equation}

\paragraph{User Access Control.} The vector $\hat{\textbf{z}}_{prot}$ consists of $N$ segments, where the $i$-th segment corresponds to $I_{sec}^i$. User $i$ attempts to decrypt all $N$ segments using its private key $\mathcal K_{priv}^i$, yielding $\hat{\textbf{z}}_{sec}(i)\in\mathbb R^{C\times H\times W}$. Crucially, since only the $i$-th segment was originally protected by $\mathcal K_{priv}^i$, only this segment's decryption yields a meaningful latent representation. This difference in decryption quality across segments implements the access control function.

\paragraph{Backward Diffusion and Image Reconstruction.} We utilize a joint denoising approach on the full vector $\mathbf{\hat{z}}_{sec}{(i)}$ through the DDIM backward process, conditioned on the null-text prompt. The resulting vector is then segmented. This method demonstrates superior performance compared to segmenting first and then applying the DDIM backward process to each segment. The detailed description of this diffusion policy is discussed in {Appendix \ref{appendix:diff_rule}}. 

The correctly decrypted $i$-th segment, $\hat{\textbf{z}}_{sec}^i{(i)}\in \mathbb R^{C\times \frac H {N_1}\times \frac W {N_2}}$ serves as an approximation of the true secret latent vector $\textbf{z}_{sec}^i$. This $\hat{\textbf{z}}_{sec}^i{(i)}$ is passed into the VAE decoder and upsampling block to obtain the final reconstructed image, $\hat I_{sec}^i{(i)}$. If user $i$ attempts to reveal another user's image using the non-decrypted segment $\hat{\textbf{z}}_{sec}^j{(i)}\in \mathbb R^{C\times \frac H {N_1}\times \frac W {N_2}}$ $(j\neq i)$, the resulting image $\hat I_{sec}^j{(i)}$ will exhibit significant degradation, as $\hat{\textbf{z}}_{sec}^j{(i)}$ significantly deviates from the true latent vector $\textbf{z}_{sec}^j$.

\subsection{Reference Generator}
\label{sec:RefGen}

In this subsection, we provide the details of RefGen. RefGen is a pre-trained diffusion model that takes the public resources ($\mathcal K_{pub}$ and $\mathcal P_{pub}$) as input to generate the reference image $I_{ref}$. The process begins by using $\mathcal K_{pub}$ to generate the initial Gaussian noise. Since the diffusion process is deterministic when the random seed (here, $\mathcal K_{pub}$) is fixed, the diffusion model always generates the identical image $I_{ref}$ from the same $\mathcal K_{pub}$ and $\mathcal P_{pub}$. This deterministic property is critical, as it eliminates the need to explicitly transmit $I_{ref}$, allowing both the sender and the designated recipient to locally reproduce $I_{ref}$.

Prior work \cite{DiffStega} leverages ControlNet \cite{controlnet} to condition the $I_{ref}$ generation using a control image (e.g., the OpenPose bone image or semantic segmentation map) extracted from the secret image. This methodology, however, presents two critical limitations. First, note that the control image must be shared publicly. Thus, publicly transmitting the structural control image compromises the covertness by exposing secret image information. Second, the rigid conditioning significantly restricts the generated image's diversity and flexibility, often leading to poor quality when the control image and prompt are semantically inconsistent. To address these issues, we deliberately omit the use of control images in MIDAS. {Instead, we achieve the high-quality conditional generation required for steganography by incorporating a Random Basis along with a Latent Vector Fusion mechanism.}
The combination of Random Basis and Latent Vector Fusion allows us to effectively embed the secret information while ensuring high fidelity, thereby eliminating the need for the external, restrictive conditioning provided by control images.

\subsection{Security Analysis}
We theoretically demonstrate that the proposed Random Basis mechanism can effectively control information leakage. For simplicity, we consider the case of single-image setting, where the sender transmits a secret image $I_{sec}$ and an attacker produces a reconstruction $\hat I_{sec}$. Following standard information-theoretic definitions \cite{NIT}, we define the information leakage rate as a security measure:
\begin{equation}
    R_L=\frac{1}{m}I(I_{sec};\hat I_{sec})
\end{equation}
where $I(\cdot;\cdot)$ denotes mutual information and $m$ is the dimension of $I_{sec}$. 

We establish the following theorem.
\begin{theorem}
    \label{thm:random_basis}
    Assume that all variables are quantized with sufficiently small step size $\Delta$. For sufficiently large $m$, the information leakage under the Random Basis mechanism with strength $\gamma>0$ satisfies
 \begin{equation}
        R_L 
        \approx O\left( \frac{-\log \Delta+ \log {m}}{{m}}+(1-\gamma)(-\log\Delta+1)\right). \label{leakage}
    \end{equation}
\end{theorem}
In practice, $\Delta \approx 10^{-7}$ (e.g., float32 precision), implying that $-\log \Delta$ is on the order of $10$. For a typical $512 \times 512 \times 3$ image, the dimension is on the order of $m \approx 10^6$. Hence, the first term on the RHS of~\eqref{leakage} becomes negligible. The second term decreases as $\gamma$ approaches $1$. Empirically, even with $\gamma = 0.4$, the proposed method provides sufficient security (see Table~\ref{tab:main}).

\begin{proof}[Sketch of proof]
  By the data processing inequality, $I(I_{sec};\hat I_{sec})$ is upper-bounded by the mutual information between the noisy latent representations of $I_{sec}$ and $\hat I_{sec}$. Therefore, securing the noisy latent representation is sufficient to control the overall information leakage. To reflect practical settings while maintaining analytical tractability, we consider quantized noisy latent representations. Due to the high dimensionality of the latent space, we leverage tools from high-dimensional probability to characterize the scaling behavior of the mutual information with respect to $\Delta$ and $m$, which leads to the stated bound. The detailed proof is provided in Appendix~\ref{appendix:proof}.
\end{proof}

\section{Experiment}
\label{sec:experiment}
\subsection{Experimental Settings}
\label{sec:exp setting}
\paragraph{Implementation.} Consistent with prior work \cite{DiffStega}, we utilize Stable Diffusion v1.5\footnote{https://huggingface.co/runwayml/stable-diffusion-v1-5} for both the forward diffusion in the Hiding Stage and backward diffusion in the Reconstruction Stage. PicX\_real\footnote{https://huggingface.co/GraydientPlatformAPI/picx-real} is employed for RefGen. Further details are provided in Appendix \ref{appendix:impl_detail}. 

\paragraph{Datasets.} We employ the Stego260 \cite{CRoSS} and UniStega \cite{DiffStega} datasets, which are specifically designed for diffusion steganography. These datasets include the images with their text descriptions (source caption) and semantically modified text descriptions (target caption). We only use the target caption for $\mathcal P_{pub}$ and a null-text is used for private prompt. These datasets are directly adopted for the $N=1$ case. To align the dataset with our multi-image hiding scenario (which requires $N$ secret images per stego image), we randomly sample $N$ images as secret images and independently select a single text prompt from the entire set of prompts in the dataset. Unless otherwise specified, all experiments are conducted with $N=2$.

\paragraph{Metrics.} We evaluate image hiding and reconstruction quality using Peak Signal-to-Noise Ratio (PSNR), Structural Similarity Index (SSIM) \cite{SSIM}, and Learned Perceptual Image Patch Similarity (LPIPS) \cite{LPIPS}. The CLIP Score \cite{CLIP_score} measures the semantic alignment between $I_{stego}$ and $\mathcal P_{pub}$. In addition, Multi-dimension Attention Network for no-reference Image Quality Assessment (MANIQA) \cite{MANIQA} is employed to assess the perceived visual naturalness of the generated stego images, reflecting human judgment. All quantitative results are reported as the average over three independent trials with different random seeds.

\begin{figure*}[ht!]
\begin{center}
\center{\includegraphics[width=6.5in]{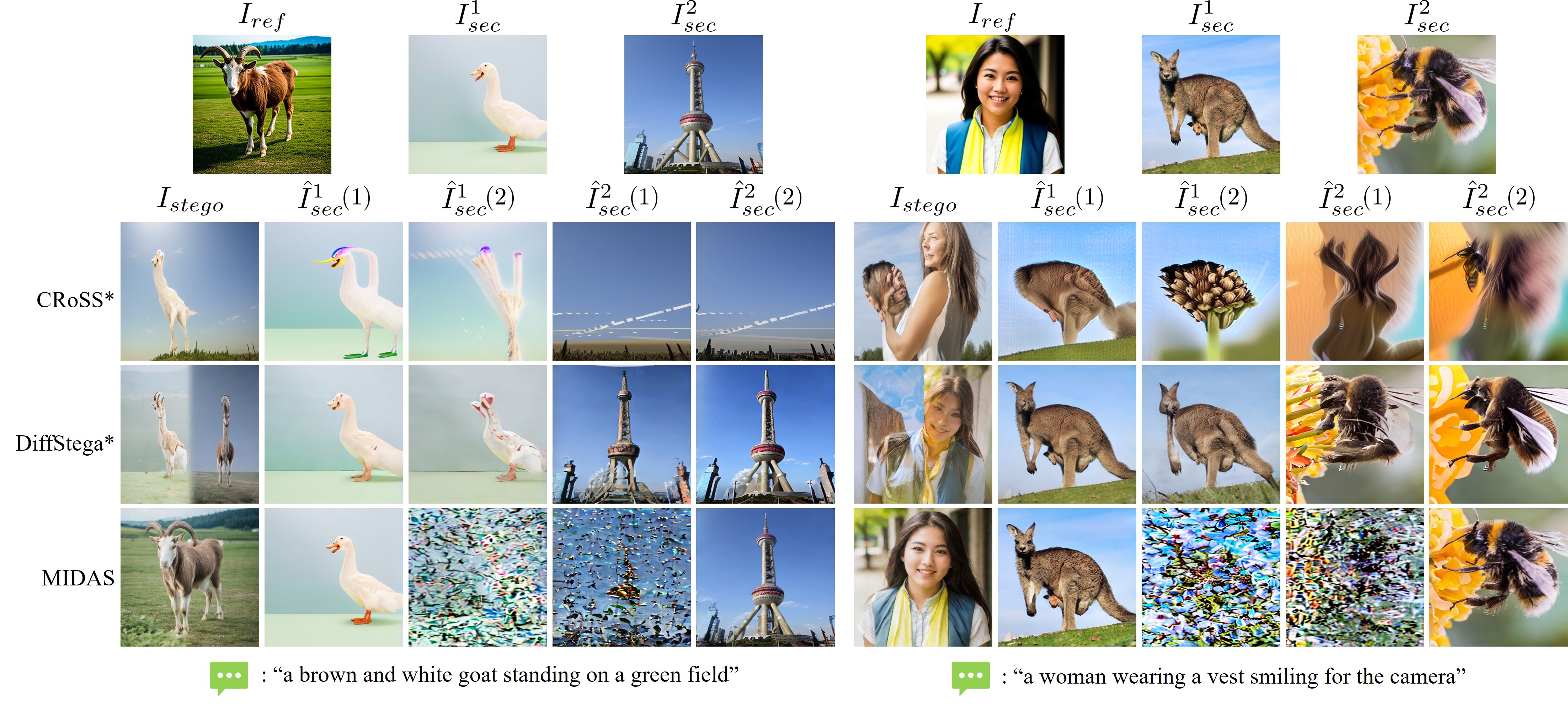}}
\caption{Visual comparisons of MIDAS, CRoSS*, and DiffStega* on the Stego260 dataset under different prompts for two-image hiding.}
\label{qual}
\end{center}
\vskip -0.2in
\end{figure*}
\paragraph{Baselines.} For training-free CIS baselines, we consider CRoSS \cite{CRoSS} and DiffStega \cite{DiffStega}, {detailed} in Section \ref{sec:CIS}. To evaluate the multi-image hiding scenario (i.e., $N>1$), we adapted CRoSS and DiffStega into CRoSS* and DiffStega*, respectively. For CRoSS*, we utilized source captions as private keys and concatenated the noisy latents of the secret images. DiffStega* was modified by applying Noise Flip with $\mathcal K_{priv}$ to each noisy latent before concatenation, followed by an additional Noise Flip with $\mathcal K_{pub}$ on the resulting vector. For both CRoSS* and DiffStega*, the public prompt $\mathcal{P}_{pub}$ is employed to generate the stego images. {While} the original DiffStega utilizes control images, we {omit} them in DiffStega* for the reasons discussed in Section \ref{sec:RefGen}. For consistency, the names CRoSS* and DiffStega* are used throughout the paper, even for the $N=1$ case where they remain functionally identical to the original models. Detailed implementation settings are provided in Appendix \ref{appendix:impl_detail}.

We maintained the original parameters specified in the baseline studies for both CRoSS* and DiffStega*. In contrast, we also evaluated these baselines using the {settings adjusted to MIDAS} within our ablation studies (see {Appendix} \ref{sec:eff_LVF}). This allowed us to assess the contribution of Latent Vector Fusion, validating their superiority over the conventional approaches even under identical experimental conditions.


For access-controlled image steganography baselines, we consider IIS \cite{IIS} and AIS \cite{AIS}, even though they are modification-based (neither coverless nor training-free) methods.

\subsection{Experimental Results}
\label{sec:exp results}
\paragraph{Qualitative Results.} Fig. \ref{qual} provides a visual comparison among training-free CIS methods. Due to the limited editing capability of CRoSS* and DiffStega*, stego images generated by these methods exhibit noticeable visual fragmentation and degraded quality. Furthermore, when extended to the multi-image scenario, CRoSS* struggles significantly to reconstruct the secret images. While DiffStega* shows relatively better reconstruction with the correct $\mathcal K_{priv}$ (e.g., $\hat I_{sec}^1(1), \hat{I}_{sec}^2(2)$), it fails to enforce access control, as meaningful reconstructions are still achievable even with incorrect $\mathcal K_{priv}$ (e.g., $\hat I_{sec}^1(2), \hat{I}_{sec}^2(1)$). In contrast, MIDAS produces high-fidelity stego images and demonstrates robust access control; it provides superior reconstruction quality for authorized users while yielding reconstruction failures when an incorrect $\mathcal K_{priv}$ is applied. More results are available in Appendix \ref{appendix:4in1}. 

\begin{table*}[ht!]
\vskip 0.1in
\caption{Quantitative results for stego and reconstructed images on the Stego260 dataset. {Best results among training-free CIS methods are highlighted in \textbf{bold}.}}
\label{tab:main}
\begin{center}
\begin{small}
\begin{sc}
\setlength{\tabcolsep}{1pt}
\begin{tabular*}{\textwidth}{cc ccc cc ccc ccc ccc}
\toprule
\multirow{2}{*}{\normalfont $N$} & 
\multirow{2}{*}{\normalfont Method} & 
\multicolumn{1}{c}{\normalfont Stego image quality} & 
\multicolumn{4}{c}{\normalfont Stego image diversity} & 
\multicolumn{3}{c}{\normalfont Correct $\mathcal K_{priv}$ reconstruction} & 
\multicolumn{3}{c}{\normalfont Wrong $\mathcal K_{priv}$ reconstruction} &  \\
\cmidrule(lr){3-3} \cmidrule(lr){4-7} \cmidrule(lr){8-10} \cmidrule(lr){11-13} \cmidrule(lr){14-16}
& &\normalfont \normalfont MANIQA$\uparrow$ &PSNR$\downarrow$ & \normalfont SSIM$\downarrow$ & \normalfont LPIPS$\uparrow$ & \normalfont CLIP Score$\uparrow$ &  \,\,\,\normalfont PSNR$\uparrow$ & \,\,\,\,\,\normalfont SSIM$\uparrow$ \,& \normalfont LPIPS$\downarrow$ \!& \normalfont \,\,\,PSNR$\downarrow$ \,\,\,& \normalfont SSIM$\downarrow$ \!\!& \normalfont LPIPS$\uparrow$\\
\midrule
\multirow{5}{*}{1} 
& \normalfont IIS & - & - & - & - & - & 48.588 & 0.999 & 0.016 & 10.551 & 0.107 & 0.818\\
& \normalfont AIS & - & - & - & - & - & 35.990 & 0.999 & 0.127 & 22.981 & 0.595 & 0.367\\
\cmidrule(lr){2-16}
&{\normalfont CRoSS*} & 0.409 & 20.535 & 0.740 & 0.322 & 27.937 & 22.870 & 0.796 & 0.263 & 18.387 & 0.650 & 0.430\\
& \normalfont DiffStega* & \textbf{0.450} & {19.686} & {0.664} & {0.419} & {28.871} & {24.958} & \textbf{0.833} & \textbf{0.232} & {19.447} & {0.650} & {0.397}\\
& \normalfont MIDAS(Ours) & {0.429} & \textbf{13.407} & \textbf{0.419} & \textbf{0.610} & \textbf{29.686} & \textbf{25.161} & {0.831} & {0.234} & \textbf{12.718} & \textbf{0.237} & \textbf{0.698}\\
\hline 
\rule{0pt}{1em}
\multirow{5}{*}{2} 
& \normalfont IIS & - & - & - & - & - & 41.360 & 0.995 & 0.070 & 11.883 & 0.219 & 0.808\\
& \normalfont AIS & - & - & - & - & - & 30.724 & 0.973 & 0.317 & 4.540 & 0.126 & 0.847 \\
\cmidrule(lr){2-16}
& \normalfont CRoSS* & {0.406} & {15.550} & {0.521} & {0.579} & 26.071 & 17.606 & 0.563 & 0.496 & {15.270} & {0.470} & {0.576} \\
& \normalfont DiffStega* & 0.399 & 17.065 & 0.531 & 0.555 & {26.952} & {21.908} & {0.728} & {0.344} & 18.137 & 0.587 & 0.458\\
& \normalfont MIDAS(Ours) & \textbf{0.434} & \textbf{9.885} & \textbf{0.287} & \textbf{0.752} & \textbf{30.129} & \textbf{23.903} & \textbf{0.771} & \textbf{0.299} & \textbf{9.964} & \textbf{0.090} & \textbf{0.753}\\
\hline 
\rule{0pt}{1em}
\multirow{5}{*}{4} 
& \normalfont IIS & - & - & - & - & - & 33.540 & 0.976 & 0.184 & 15.783 & 0.338 & 0.686 \\
& \normalfont AIS & - & - & - & - & - & 28.045 & 0.956 & 0.369 & 5.131 & 0.126 & 0.947\\
\cmidrule(lr){2-16}
& \normalfont CRoSS* & {0.418} & {13.445} & {0.453} & {0.687} & 24.601 & 13.190 & 0.312 & 0.681 & {12.731} & {0.297} & {0.696}\\
& \normalfont DiffStega* & 0.364 & 16.160 & 0.509 & 0.600 & {27.367} & {19.233} & {0.609} & {0.442} & 17.533 & 0.545 & 0.508\\
& \normalfont MIDAS(Ours) & \textbf{0.479} & \textbf{8.996} & \textbf{0.296} & \textbf{0.787} & \textbf{30.169} & \textbf{22.283} & \textbf{0.697} & \textbf{0.359} & \textbf{9.399} & \textbf{0.118} & \textbf{0.763}\\
\bottomrule
\end{tabular*}
\end{sc}
\end{small}
\end{center}
\vskip -0.2in
\end{table*}
\paragraph{Quantitative Results.} Table \ref{tab:main} provides a quantitative comparison of MIDAS against baseline methods on the Stego260 dataset. Additional experimental results conducted on the UniStega dataset are provided in Appendix \ref{appendix:unistega}. We define stego image diversity as the dissimilarity {(e.g., low PSNR and SSIM, or high LPIPS) between each secret image and its corresponding region in the stego image, where the layout is partitioned by the concatenation boundaries}. The CLIP Score also contributes to stego image diversity assessment, as it reflects the model's capability for prompt-aligned image generation. High diversity is crucial for security, as it prevents third parties from predicting the secret content. We measure stego image quality using MANIQA, which provides a state-of-the-art assessment of visual naturalness that aligns closely with human perception.

CRoSS* and DiffStega* exhibit poor stego image diversity; their generated stego images remain highly similar to the secret images and fail to align with $\mathcal P_{pub}$ (i.e., low CLIP Score). This low diversity not only severely compromises security but also makes concatenation-based multi-image hiding highly challenging, as it often leads to visually fragmented stego image generation in multi-user scenarios. Consequently, the stego image quality of these baseline methods deteriorates significantly as $N$ increases; specifically, DiffStega* suffers from a sharp decline in visual quality (indicated by MANIQA values), while CRoSS* fails to generate images aligned with the prompts (indicated by extremely low CLIP Scores). In contrast, MIDAS exhibits both high stego image quality (high MANIQA) and superior diversity (extremely low similarity with the secret images) regardless of $N$. This demonstrates the suitability and scalability of MIDAS for multi-image hiding scenarios. Furthermore, evaluations under a larger number of secret images (e.g., $N=8$) are provided in Appendix \ref{appendix:large_capacity} to further validate the robustness of our framework in high-capacity settings.

Furthermore, the data confirms the robustness of MIDAS's access control functionality. Unlike CRoSS* and DiffStega*, where unauthorized reconstruction often remains possible, MIDAS exhibits a significant difference in similarity metrics between the reconstructed image and the secret image when the correct $\mathcal K_{priv}$ is used (high similarity) versus when the wrong $\mathcal K_{priv}$ is applied (low similarity). This provides quantitative validation that the access control mechanism is successfully implemented, consistent with the qualitative results (Fig. \ref{qual}). 


Although modification-based methods (IIS and AIS) show superior performance compared to training-free CIS methods, they inherently possess security vulnerabilities. More precisely, the resulting modification leaves inevitable traces in the stego image, making it easily detectable if the original cover image is exposed. Furthermore, these methods are training-based, {requiring additional optimization overhead.}

\setlength{\tabcolsep}{1pt}
\begin{table}[t]
\caption{PSNR(dB) of reconstructed images under practical degradations. For Gaussian {n}oise and JPEG {c}ompression, standard deviation $\sigma = 5$ and quality factor $Q = 70$ are used, respectively. For Gaussian blur, $\sigma = 2$ and kernel size 7 are used. Considering these degradations as additional noise, 5 additional denoising steps are applied to the diffusion-based frameworks (CRoSS*, DiffStega*, and MIDAS) in the Reconstruction Stage.}
\label{tab:degrad}
\begin{center}
\begin{small}
\fontsize{8pt}{10pt}\selectfont
\begin{sc}
\begin{tabular}{ccccc}
\toprule
\normalfont Method & \normalfont Clean & \normalfont Gaussian noise & \normalfont JPEG  & \normalfont Gaussian blur\\
\midrule
\normalfont IIS & 41.360 & 12.439 & 10.047 & 10.677\\
\normalfont AIS & 30.774 & 14.654 & 9.442 & 9.996\\
\normalfont CRoSS* & 17.606 & 16.300 & 16.932 & 15.629\\
\normalfont DiffStega* & 21.908  & 20.076 & 20.422 & 19.375\\
\normalfont MIDAS(Ours) & 23.903 & 20.046 & 19.922 & 19.689\\
\bottomrule
\end{tabular}
\end{sc}
\end{small}
\end{center}
\vskip -0.2in
\end{table}

\paragraph{Anti-Steganalysis.} Steganalysis methods aim to distinguish between stego images and cover images. We evaluate the security of MIDAS compared to baseline methods using the state-of-the-art steganalysis methods: XuNet \cite{XuNet} and SiaStegNet \cite{SiaStegNet}. These models are trained on cover-stego pairs to identify steganographic traces, with their detection accuracy generally improving as the training dataset size increases. Since coverless steganography methods lack of cover images, we utilize images generated from the same model and prompts as the corresponding cover images. Specifically, these cover images are generated using a randomly sampled noise vector under the same $\mathcal{P}_{pub}$ as the stego images. Fig. \ref{anti_steg} presents the steganalysis accuracy results across varying training data sizes. Modification-based methods (IIS and AIS) consistently achieve high steganalysis accuracy ($>90\%$), because the embedding modifications leave inherent traces within the stego images, allowing the steganalysis models to distinguish them from the original covers. Although CRoSS* and DiffStega* are coverless methods, their poor stego image quality introduces visible artifacts, causing its steganalysis accuracy to rapidly converge to a high value ($>85\%$). In contrast, the steganalysis accuracy of MIDAS is approximately 20\% lower than that of other baselines, significantly limiting the detector's confidence and demonstrating enhanced security. This resilience stems from its coverless approach combined with its superior stego image generation quality, which ensures the stego images remain highly statistically similar to genuinely generated images.

\begin{figure}[t]
\begin{center}
\center{\includegraphics[width=\columnwidth]{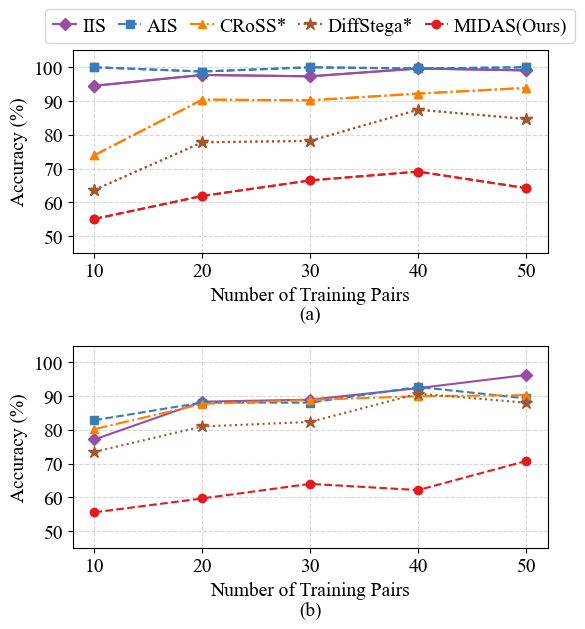}}
\vskip -0.1in
\caption{{Steganalysis accuracy on} (a) XuNet and (b) SiaStegNet.}
\label{anti_steg}
\end{center}
\vskip -0.2in
\end{figure}

\paragraph{Robustness.} We evaluate the robustness against noise introduced during the transmission of the stego image. Table~\ref{tab:degrad} reports the reconstruction quality when the stego image is corrupted by Gaussian noise or Gaussian blur, or compressed using JPEG. Even under strong degradation, MIDAS maintains an acceptable level of secret information transmission. In contrast, existing modification-based access control methods (IIS, AIS) exhibit a significant performance drop under these degradations. This superior performance highlights the robustness of MIDAS to channel noise and attack, which is a critical advantage of our diffusion-based approach.

\paragraph{More Experiments.} {Additional evaluations on a different dataset and in a high-capacity setting ($N=8$), comparisons with HIS \cite{HIS}, and analyses of practical overheads are provided in Appendix~\ref{appendix:more_ex}. Ablation studies on the key components of MIDAS, different model checkpoints, and hyperparameter choices are presented in Appendix~\ref{sec:ablation}. These results demonstrate the robustness of MIDAS across datasets, its scalability to extreme-capacity settings, and its ability to maintain computational time and memory usage comparable to baseline training-free CIS methods. They also validate the effectiveness of the proposed components and strategies across different model checkpoints, while characterizing the impact of individual hyperparameters.}


\section{Conclusion}
This paper presents MIDAS, a training-free framework for access-controlled CIS utilizing pre-trained diffusion models. By leveraging the Random Basis and Latent Vector Fusion, MIDAS achieves robust access control and high-fidelity reconstruction without the need for resource-intensive training. Extensive experiments confirm that MIDAS outperforms training-free CIS baselines in both security and access control. {Owing to its superior access control, MIDAS facilitates flexible and efficient communication in various multi-user scenarios through private key distribution {based on information access privileges}. However, as the diffusion process can involve inference latency, optimizing sampling strategies for improved efficiency could be a meaningful extension.}

\section*{Acknowledgements}
This work was supported in part by the Institute of Information \& Communications Technology Planning \& Evaluation (IITP) through the Next Generation Semantic Communication Network Research Center Grant (RS-2024-00398948), and in part by the IITP through 6G · Cloud Research and Education Open Hub Grant (IITP-2026-RS-2024-00428780), funded by the Korea government (MSIT).

\section*{Impact Statement}
In this work, we propose MIDAS, a training-free and access-controlled scheme for coverless image steganography using pre-trained diffusion models. This capability makes our approach highly effective for implementing robust access control in practical communication scenarios that preclude resource-intensive training. By enabling selective and secure data sharing, {even scalable given its multi-image hiding capability,} our algorithm holds the potential to enhance the privacy and practical utility of coverless image steganography systems in real-world applications. {However, as this technology could be exploited for privacy attacks, such as unauthorized data exfiltration, its ethical implications warrant careful attention within communities where stego-images are shared. }



\bibliography{BIB/ASDF}
\bibliographystyle{icml2026}

\newpage
\appendix
\onecolumn

\section{Architecture of MIDAS}
\label{appendix:architecture}
\begin{figure}[ht!]
\vskip 0.2in
\begin{center}
\center{\includegraphics[width=6.5in]{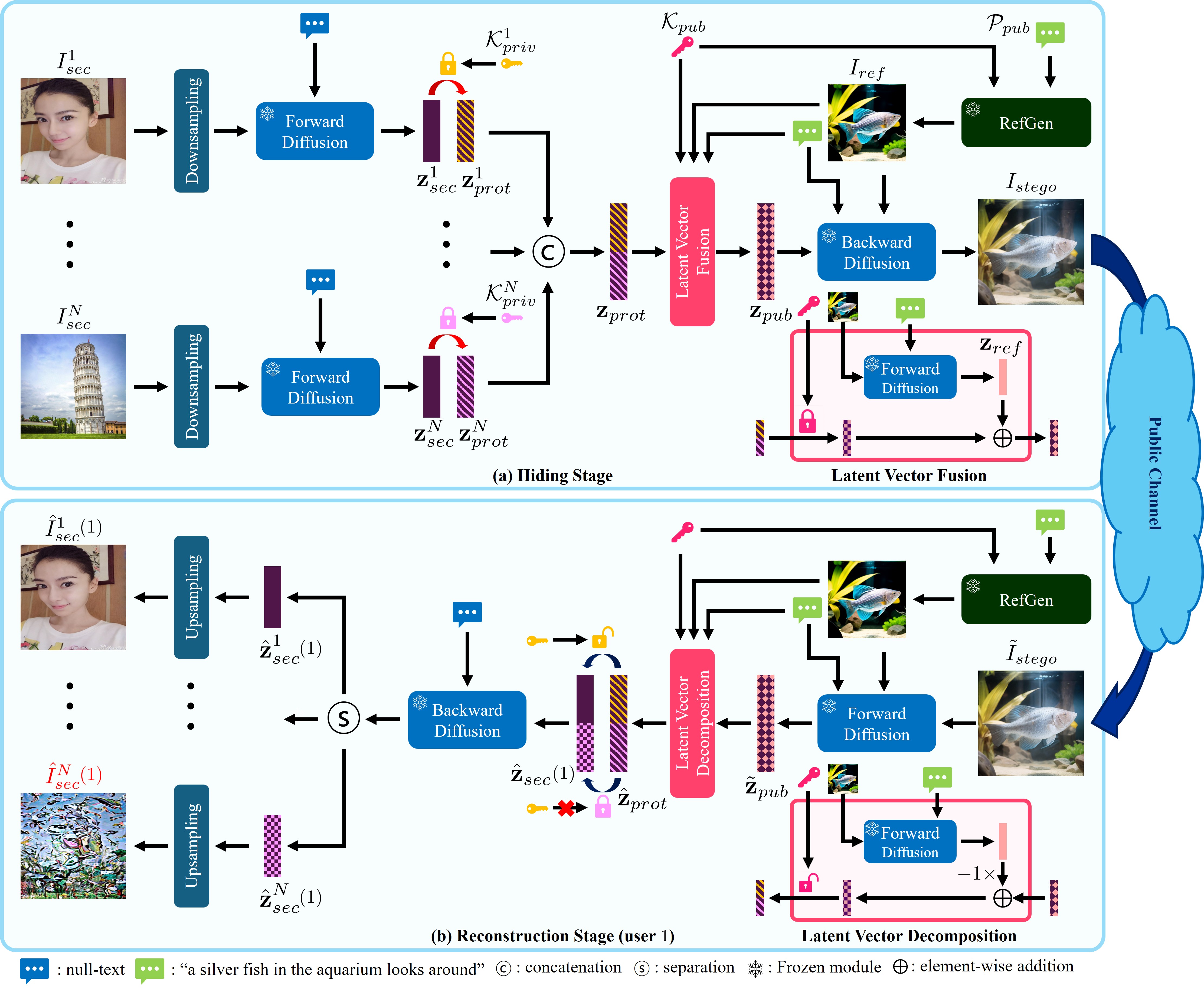}}
\caption{The overall architecture of MIDAS. (a) The Hiding Stage, where $N$ secret images (${I}_{sec}^1, ..., {I}_{sec}^N$) are individually encrypted using private keys (${\mathcal{K}}_{priv}^i$) and then jointly fused using the public key ($\mathcal{K}_{pub}$) to form the single stego image (${I}_{stego}$). (b) The Reconstruction Stage (for user $1$), where user $1$ recovers their assigned image (${\hat{I}}_{sec}^1$) using ${\mathcal{K}}_{priv}^1$, from the received stego image (${\tilde{I}}_{stego}$) (which may contain noise due to transmission errors), while the access control mechanism ensures that other images remain encrypted.}
\label{architecture}
\end{center}
\vskip -0.2in
\end{figure}
\newpage

\section{Experimental Details}
\label{appendix:exp_detail}
\subsection{Construction of Random Basis Matrix}
\label{appendix:random_basis}
To construct $Q_d$, an intermediate orthonormal matrix $Q'_{\lfloor \gamma d \rfloor} \in \mathbb{R}^{\lfloor \gamma d \rfloor \times \lfloor \gamma d \rfloor}$ is first generated via QR decomposition using a seed $\mathcal{K}$. This matrix is then extended into a larger block-diagonal form, $\text{diag}(Q'_{\lfloor \gamma d \rfloor}, I)$, where the remaining entries are filled by an identity matrix. Finally, $Q_d$ is obtained by shuffling the rows of this block-diagonal matrix, with the permutation order determined by the same seed $\mathcal{K}$.

\subsection{Implementation Details}
\label{appendix:impl_detail}
All experiments were conducted in a Python 3.12.9 environment equipped with an AMD Ryzen Threadripper PRO 3955WX (16-core) processor and an NVIDIA RTX 3090 GPU. For the diffusion model, the total time step is set to $T=50$ and the EDICT \cite{EDICT} mixing coefficient to 0.93. To ensure high-quality reconstruction, only a portion $\xi$ of the total steps is executed; the diffusion process thus runs over steps $[0, \xi T]$. We define the private diffusion stage as the phase where secret images are inverted, and the public diffusion stage as the phase where the fused latent is reconstructed into the stego image. We set $\xi=0.4$ for the private diffusion stage and the Latent Vector Fusion stage (to get $\mathbf{z}_{ref}$), and $\xi=0.7$ for the public diffusion stage. For ${N>1}$, $\gamma_{priv} = 0.4$ and $\gamma_{fuse}=0.5$ are used. For $N=1$, since latent vector mixing is unnecessary, we set $\gamma_{fuse} = 0$ and $\gamma_{priv}=0.5$. In the Latent Vector Fusion, $\alpha={0.95}$ is used. Furthermore, for $N>1$, a 5-step DDIM \cite{DDIM} forward process is applied immediately following Private Key Encryption, followed by a 5-step DDIM backward process prior to the User Access Control stage. We empirically found that this intermediate DDIM sequence significantly enhances reconstruction quality. The source code is available at https://github.com/Minyeol/MIDAS.

To evaluate CRoSS* \cite{CRoSS} and DiffStega* \cite{DiffStega} in a multi-user environment, we extend their official implementations by incorporating the latent concatenation mechanism utilized in MIDAS. For all the other hyperparameters (e.g., diffusion steps and key strength), we follow the default settings provided in their respective original papers.

CRoSS*: For a set of secret images $\{I_{sec}^1, \dots, I_{sec}^N\}$, each image is independently processed via a DDIM forward pass using its specific private prompt. The resulting latent vectors are then concatenated and denoised using a public prompt to generate $I_{stego}$. In our implementation, we adopted separate denoising for each secret image, as this approach was found to yield superior results for this baseline.

DiffStega*: Each secret image $\{I_{sec}^1, \dots, I_{sec}^N\}$ undergoes an individual DDIM forward pass with a null-text prompt, followed by encryption using distinct private keys $K_{priv}$ via the Noise Flip mechanism. These encrypted latents are concatenated, and an additional Noise Flip operation is applied to the combined vector. The final latent is denoised with a public prompt and a reference image to produce $I_{stego}$. The reconstruction is performed using a joint denoising strategy.

While the original DiffStega was optimized for single-image scenarios, we introduce DiffStega**, which adopts the same diffusion steps and key strengths as MIDAS to ensure a fair comparison (see {Appendix} \ref{sec:eff_LVF} for details).

DiffStega**: We set the private key strength to $\gamma_{priv} = 0.2$ and the public key strength to $\gamma_{fuse} = 0.25$, while also applying the same 5-step intermediate DDIM sequence used in MIDAS. Under these settings, the signal-to-noise ratio between the original and encrypted latents aligns with that of Random Basis encryption, ensuring an equitable comparison.

We utilize the source caption as the private prompt for CRoSS*, whereas DiffStega*, DiffStega** and MIDAS exclusively employ null-text prompts. For reconstruction under a wrong $K_{priv}$ where $N > 1$, we evaluate each model by attempting decryption with the $K_{priv}$ of every other authorized user. In the $N = 1$ case, where no other users exist, the protocols differ: CRoSS* is tested using a null-text prompt as the incorrect key, while the other training-free CIS models utilize a mismatched random seed.

For modification-based baselines, IIS \cite{IIS} and AIS \cite{AIS} were retrained using their official open-source implementations on the DIV2K dataset \cite{DIV2K}.

For steganalysis, we randomly sample pairs of cover and stego images, splitting them into a training set (80\%) and a validation set (20\%). We employ XuNet \cite{XuNet} and SiaStegNet \cite{SiaStegNet} using available open-source implementations.\footnote{https://github.com/albblgb/Deep-Steganalysis} Both models are trained for 50 epochs, and the version achieving the highest validation accuracy is selected for evaluation.

\newpage
\section{More Experiments}
\label{appendix:more_ex}
\subsection{Qualitative Results of {H}iding {F}our {I}mages}
\label{appendix:4in1}
\begin{figure*}[ht!]
\vskip 0.2in
\begin{center}
\center{\includegraphics[width=6.5in]{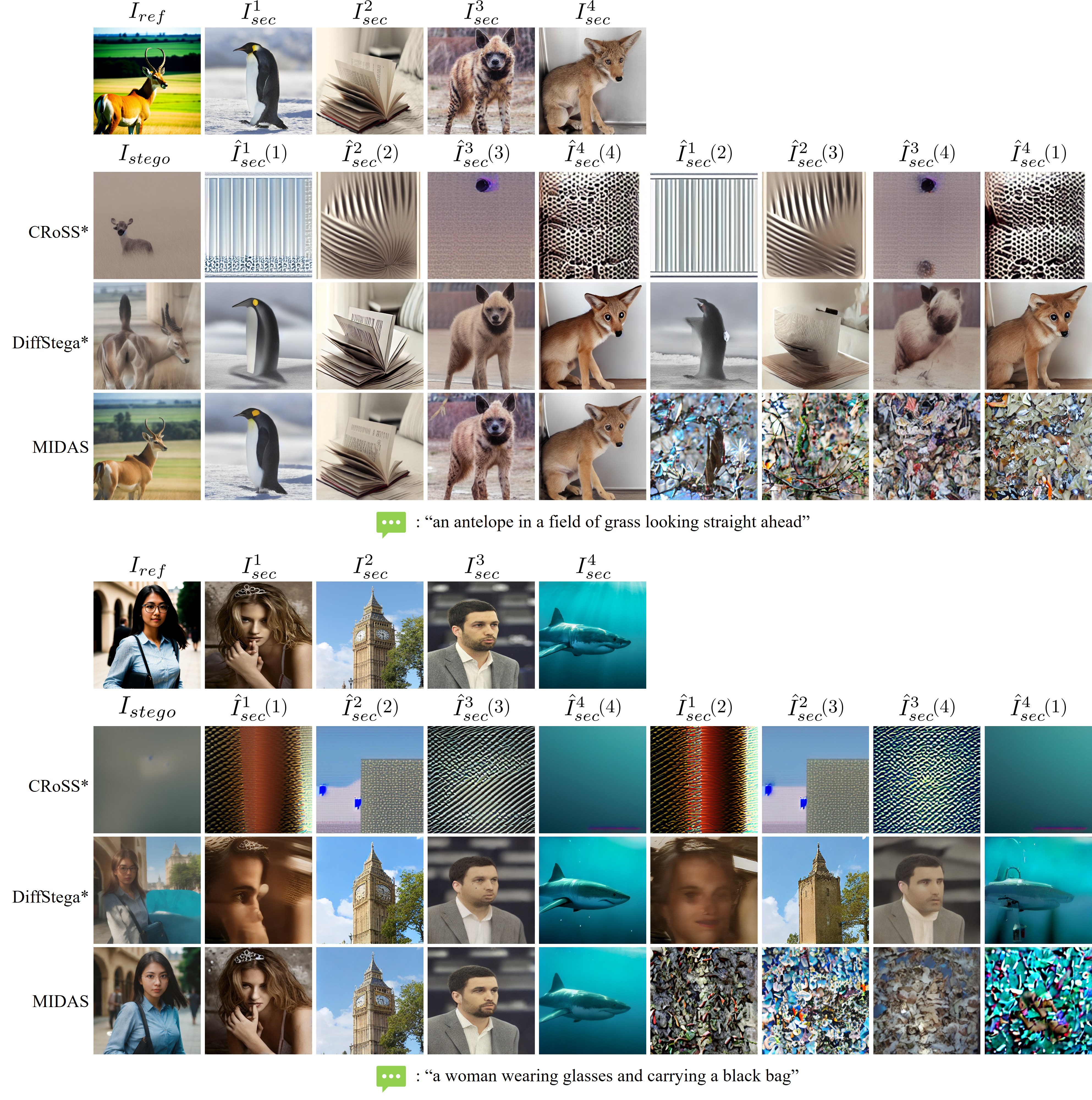}}
\caption{Visual comparisons of MIDAS{, CRoSS*,} and DiffStega{*} on the Stego260 dataset under different prompts for four-image hiding.}
\label{4in1}
\end{center}
\vskip -0.2in
\end{figure*}

\newpage
\subsection{Experimental Results on UniStega}
\label{appendix:unistega}
\begin{table*}[ht!]
\caption{Quantitative results for stego and reconstructed images on the UniStega dataset. }
\label{tab:unistega}
\begin{center}
\begin{small}
\begin{sc}
\setlength{\tabcolsep}{1pt}
\begin{tabular*}{\textwidth}{cc ccc cc ccc ccc ccc}
\toprule
\multirow{2}{*}{\normalfont $N$} & 
\multirow{2}{*}{\normalfont Method} & 
\multicolumn{1}{c}{\normalfont Stego image quality} & 
\multicolumn{4}{c}{\normalfont Stego image diversity} & 
\multicolumn{3}{c}{\normalfont Correct $\mathcal K_{priv}$ reconstruction} & 
\multicolumn{3}{c}{\normalfont Wrong $\mathcal K_{priv}$ reconstruction} &  \\
\cmidrule(lr){3-3} \cmidrule(lr){4-7} \cmidrule(lr){8-10} \cmidrule(lr){11-13} \cmidrule(lr){14-16}
& &\normalfont \normalfont MANIQA$\uparrow$ &PSNR$\downarrow$ & \normalfont SSIM$\downarrow$ & \normalfont LPIPS$\uparrow$ & \normalfont CLIP Score$\uparrow$ &  \,\,\,\normalfont PSNR$\uparrow$ & \,\,\,\,\,\normalfont SSIM$\uparrow$ \,& \normalfont LPIPS$\downarrow$ \!& \normalfont \,\,\,PSNR$\downarrow$ \,\,\,& \normalfont SSIM$\downarrow$ \!\!& \normalfont LPIPS$\uparrow$\\
\midrule
\multirow{5}{*}{1} 
& \normalfont IIS & - & - & - & - & - & 42.724 & 0.995 & 0.032 & 10.497 & 0.125 & 0.786 \\
& \normalfont AIS & - & - & - & - & - & 31.585 & 0.997 & 0.155 & 22.484 & 0.634 & 0.342\\
\cmidrule(lr){2-16} 
&{\normalfont CRoSS*} & \textbf{0.445} & 19.452 & 0.646 & 0.397 & 27.574 & 21.116 & {0.713} & {0.322} & {17.415} & {0.564} & {0.480} \\
& \normalfont DiffStega* & {0.439} & {18.450} & {0.581} & {0.471} & {28.676} & \textbf{23.183} & \textbf{0.767} & \textbf{0.269} & 18.614 & 0.589 & 0.429 \\
& \normalfont MIDAS(Ours) & 0.409 & \textbf{13.329} & \textbf{0.380} & \textbf{0.658} & \textbf{29.127} & {21.196} & 0.677 & 0.336 & \textbf{13.236} & \textbf{0.218} & \textbf{0.680} \\
\hline 
\rule{0pt}{1em}
\multirow{5}{*}{2} 
& \normalfont IIS & - & - & - & - & - & 36.053 & 0.984 & 0.104 & 11.707 & 0.244 & 0.776 \\
& \normalfont AIS & - & - & - & - & - & 28.651 & 0.973 & 0.302 & 4.573 & 0.098 & 0.856 \\
\cmidrule(lr){2-16}  
& \normalfont CRoSS* & {0.442} & {15.336} & 0.468 & {0.597} & 25.912 & 17.196 & 0.519 & 0.524 & {15.469} & {0.446} & {0.584} \\
& \normalfont DiffStega* &  0.413 & 16.380 & {0.459} & {0.592} & {27.550} & {20.638} & {0.654} & {0.393} & 17.471 & 0.519 & 0.498\\
& \normalfont MIDAS(Ours) & \textbf{0.458} & \textbf{10.431} & \textbf{0.254} & \textbf{0.760} & \textbf{29.974} & \textbf{22.563} & \textbf{0.701} & \textbf{0.346} & \textbf{10.104} & \textbf{0.092} & \textbf{0.751} \\
\hline 
\rule{0pt}{1em}
\multirow{5}{*}{4} 
& \normalfont IIS & - & - & - & - & - & 30.819 & 0.959 & 0.215 & 14.910 & 0.348 & 0.666\\
& \normalfont AIS & - & - & - & - & - & 25.916 & 0.945 & 0.365 & 5.259 & 0.102 & 0.979 \\
\cmidrule(lr){2-16} 
& \normalfont CRoSS* & {0.385} & {13.177} & {0.377} & {0.691} & 23.313 & 12.567 & 0.247 & 0.709 & {12.321} & {0.237} & {0.718} \\
& \normalfont DiffStega* & 0.360 & 15.657 & 0.438 & 0.632 & {27.238} & {18.408} & {0.538} & {0.491} & 16.975 & 0.478 & 0.550 \\
& \normalfont MIDAS(Ours) & \textbf{0.483} & \textbf{10.408} & \textbf{0.298} & \textbf{0.780} & \textbf{29.813} & \textbf{21.557} & \textbf{0.635} & \textbf{0.399} & \textbf{10.987} & \textbf{0.136} & \textbf{0.739} \\
\bottomrule
\end{tabular*}
\end{sc}
\end{small}
\end{center}
\vskip -0.1in
\end{table*}

\subsection{Very Large Capacity}
\label{appendix:large_capacity}
\begin{table*}[ht!]
\caption{Quantitative results of MIDAS on stego and reconstructed images for eight-image hiding. }
\label{tab:large_capacity}
\begin{center}
\begin{small}
\begin{sc}
\setlength{\tabcolsep}{1pt}
\begin{tabular*}{\textwidth}{cc ccc cc ccc ccc ccc}
\toprule
\multirow{2}{*}{\normalfont $N$} & 
\multirow{2}{*}{\normalfont \quad Dataset\quad} & 
\multicolumn{1}{c}{\normalfont Stego image quality} & 
\multicolumn{4}{c}{\normalfont Stego image diversity} & 
\multicolumn{3}{c}{\normalfont Correct $\mathcal K_{priv}$ reconstruction} & 
\multicolumn{3}{c}{\normalfont Wrong $\mathcal K_{priv}$ reconstruction} &  \\
\cmidrule(lr){3-3} \cmidrule(lr){4-7} \cmidrule(lr){8-10} \cmidrule(lr){11-13} \cmidrule(lr){14-16}
& &\normalfont MANIQA$\uparrow$ &PSNR$\downarrow$ & \normalfont SSIM$\downarrow$ & \normalfont LPIPS$\uparrow$ & \normalfont CLIP Score$\uparrow$ &  \,\,\,\normalfont PSNR$\uparrow$ & \,\,\,\,\,\normalfont SSIM$\uparrow$ \,& \normalfont LPIPS$\downarrow$ \!& \normalfont \,\,\,PSNR$\downarrow$ \,\,\,& \normalfont SSIM$\downarrow$ \!\!& \normalfont LPIPS$\uparrow$\\
\midrule
\multirow{2}{*}{8} 
& \normalfont \quad Stego260 \quad & 0.456 & 9.698 & 0.357 & 0.785 & 30.257 & 20.502 & 0.624 & 0.440 & 11.219 & 0.187 & 0.731\\
& \normalfont \quad UniStega \quad & 0.421 & 10.320 & 0.316 & 0.799 & 29.084 & 19.439 & 0.543 & 0.502 & 11.230 & 0.167 & 0.738 \\
\bottomrule
\end{tabular*}
\end{sc}
\end{small}
\end{center}
\vskip -0.1in
\end{table*}

\begin{figure*}[ht!]
\vskip 0.2in
\begin{center}
\center{\includegraphics[width=6.5in]{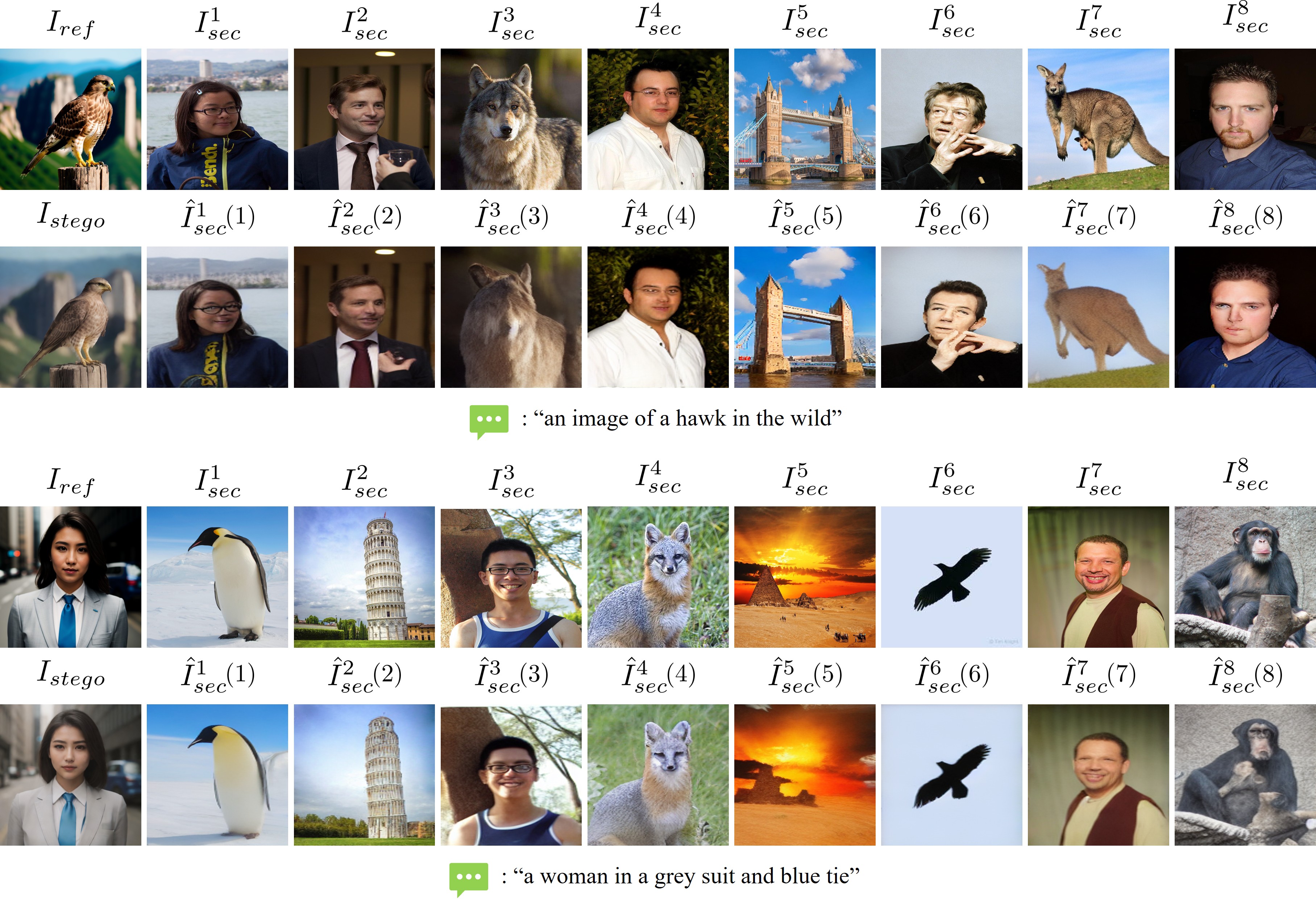}}
\caption{Qualitative results of MIDAS on the Stego260 dataset under different prompts for eight-image hiding.}
\label{8in1}
\end{center}
\vskip -0.2in
\end{figure*}
\clearpage
\newpage

\subsection{Computational Complexity and Memory Usage}
\label{appendix:complexity}
\begin{table}[ht]
\caption{Runtime and peak memory usage of MIDAS and baseline methods when processing 50 secret images.}
\label{tab:complexity}
\begin{center}
\begin{small}
\begin{sc}
\begin{tabular}{ccc}
\toprule
\normalfont Method & \normalfont Time (s)$\downarrow$ & \normalfont Peak GPU$\downarrow$\\
\midrule
\normalfont IIS & 1.540 & 0.811\\
\normalfont AIS & 0.257 & 0.337\\
\normalfont CRoSS* & 38.982 & 8.263\\
\normalfont DiffStega* & 28.436 & 6.944\\
\normalfont MIDAS(Ours) & 33.867 & 7.371\\
\bottomrule
\end{tabular}
\end{sc}
\end{small}
\end{center}
\vskip -0.1in
\end{table}
MIDAS has comparable resource requirements to other diffusion-based methods (e.g., CRoSS* and DiffStega*), with $\sim30$~s inference time on an AMD Ryzen Threadripper PRO 3955WX and an NVIDIA RTX 3090. The peak GPU memory of $\sim 7$~GB makes it feasible on mid-range GPUs (e.g., NVIDIA RTX 3050 (8 GB)). Modification-based methods (e.g., IIS, AIS) are significantly faster and more memory-efficient, but generally more vulnerable to steganalysis (Fig. \ref{anti_steg}).

\subsection{Comparison with HIS}
HIS \cite{HIS} extends multi-image hiding to the CIS setting by reusing stego images generated by CRoSS \cite{CRoSS} as cover images for conventional modification-based multi-image hiding methods. HIS consists of two main components: an Enhance-Flow that improves the reconstruction quality of DDIM, and an Embed-Flow that hides additional content into the CRoSS-generated stego images using modification-based techniques. Due to its complex pipeline, naive access control extensions, as used in CRoSS and DiffStega, are not directly applicable; thus, we provide a separate comparison.

Since HIS is a training-based method unlike the training-free MIDAS, we introduce a simplified variant, HIS*, to enable a fair comparison in terms of computational cost. Specifically, HIS* replaces the Enhance-Flow with an exact inversion method (e.g., EDICT \cite{EDICT}) and employs an Embed-Flow trained on real images (e.g., DIV2K \cite{DIV2K}) instead of diffusion-generated images. The experimental results are shown in Table \ref{tab:HIS*}, where MIDAS consistently achieves substantially higher stego image diversity and demonstrates superior reconstruction and stego image quality as $N$ increases.

\label{appendix:HIS*}
\begin{table*}[ht!]
\caption{Quantitative results for stego and reconstructed images of HIS* and MIDAS on the Stego260 dataset.}
\label{tab:HIS*}
\begin{center}
\begin{small}
\begin{sc}
\setlength{\tabcolsep}{6pt}
\begin{tabular}{cc ccc cc ccc ccc}
\toprule
\multirow{2}{*}{\normalfont $N$} & 
\multirow{2}{*}{\normalfont Method} & 
\multicolumn{1}{c}{\normalfont Stego image quality} & 
\multicolumn{4}{c}{\normalfont Stego image diversity} & 
\multicolumn{3}{c}{\normalfont Correct $\mathcal K_{priv}$ reconstruction}\\
\cmidrule(lr){3-3} \cmidrule(lr){4-7} \cmidrule(lr){8-10} \cmidrule(lr){11-13}
& &\normalfont \normalfont MANIQA$\uparrow$ &PSNR$\downarrow$ & \normalfont SSIM$\downarrow$ & \normalfont LPIPS$\uparrow$ & \normalfont CLIP Score$\uparrow$ & \normalfont PSNR$\uparrow$ & \normalfont SSIM$\uparrow$ & \normalfont LPIPS$\downarrow$\\
\midrule
\multirow{2}{*}{1} 
& \normalfont HIS* & 0.429 & 23.515 & 0.819 & 0.224 & 27.195 & 29.686 & 0.847 & 0.195\\
& \normalfont MIDAS(Ours) & 0.429 & 13.407 & 0.419 & 0.610 & 29.686 & 25.161 & 0.831 & 0.234 \\
\hline 
\rule{0pt}{1em}
\multirow{2}{*}{2} 
& \normalfont HIS* & 0.394 & 22.712 & 0.816 & 0.490	& 27.304 & 30.149 & 0.911 & 0.176\\
& \normalfont MIDAS(Ours) & 0.434 & 9.885 & 0.287 & 0.752 & 30.129 & 23.903 & 0.771 & 0.299 \\
\hline 
\rule{0pt}{1em}
\multirow{2}{*}{4} 
& \normalfont HIS* & 0.366 & 22.920 & 0.814 & 0.485 & 27.245 & 19.347 & 0.792 & 0.442\\
& \normalfont MIDAS(Ours) & 0.479 & 8.996 & 0.296 & 0.787 & 30.169 & 22.283 & 0.697 & 0.359\\
\bottomrule
\end{tabular}
\end{sc}
\end{small}
\end{center}
\vskip -0.1in
\end{table*}

\newpage
\section{Ablation Study}
\label{sec:ablation}

\subsection{Effectiveness of Encryption Strategy.}
\label{appendix:key_mechanism}
As shown in prior work \cite{zigzag}, residual structural information in the latent vector can degrade the diffusion model's generation quality when the structural information is poorly aligned with the generative prompt. Noise Flip, which only inverts the signs of the latent values, inherently preserves some spatial information from the original image. In contrast, Random Basis alters the latent vector by modifying the actual values, thereby more effectively scrambling and eliminating such spatial information.

To quantitatively verify this effect, we compare the amount of residual structural information remaining after applying Random Basis versus Noise Flip.  We use an extremely high key strength to maximize the observable effect to the mechanisms. The evaluation protocol is as follows: 
\begin{enumerate}
    \item  Apply DDIM forward diffusion to the original image. 
    \item  Apply the key mechanism to $K$ channels out of the four diffusion latent vector channels. Increasing $K$ lowers the structural similarity between the pre- and post-encryption latent vectors.
    \item Generate the encrypted image using DDIM backward diffusion on the encrypted vector.
\end{enumerate}
We measure the structural similarity between the original image and the encrypted image in terms of the structural component ($S$) of SSIM defined as follows: 
\begin{equation}
    S(\textbf{x},\textbf{y}) = \frac{2\sigma_{\textbf{xy}} + C}{\sigma_\textbf{x}^2 + \sigma_\textbf{y}^2 + C}
\end{equation}
where $\textbf{x}$ and $\textbf{y}$ are input image patches, $\sigma_\textbf{x}^2$ and $\sigma_\textbf{y}^2$ are the local variances, $\sigma_\textbf{xy}^2$ is the local covariance, and $C$ is a stabilization constant. Note that a lower value of $S$ indicates a lower structural similarity. Table \ref{tab:key_mechanism} shows that Random Basis consistently achieves lower structural similarity than Noise Flip, regardless of the number of encrypted channels.

\setlength{\tabcolsep}{3.5pt}
\begin{table}[ht]
\caption{Quantification of residual structural information in the stego image, measured by the structural component of SSIM between $I_{sec}$ and $I_{stego}$.}
\label{tab:key_mechanism}
\begin{center}
\begin{small}
\begin{sc}
\begin{tabular}{cccccc}
\toprule
\multirow{2}{*}{\normalfont Method} & 
\multicolumn{5}{c}{\normalfont Number of encrypted channels}\\
\cmidrule(lr){2-6}
& 0 & 1 & 2 & 3 & 4\\
\midrule
\normalfont Noise Flip    & 0.809 & 0.582 & 0.355 & 0.201 & 0.089 \\
\normalfont Random Basis  & 0.809 & 0.560 & 0.286 & 0.128 & 0.048 \\
\bottomrule
\end{tabular}
\end{sc}
\end{small}
\end{center}
\vskip -0.1in
\end{table}

\subsection{Effectiveness of Diffusion Strategy.}
\label{appendix:diff_rule}
In the Reconstruction Stage of MIDAS, each user $i$ employs joint denoising of the full latent vector $\hat{\mathbf{z}}_{sec}(i)$ instead of denoising the individual segments $\hat{\mathbf{z}}_{sec}^1(i), ..., \hat{\mathbf{z}}_{sec}^N(i)$ separately. As shown in Table \ref{tab:diff_rule}, this joint denoising approach significantly outperforms separate denoising during reconstruction, yielding a PSNR$_\text{Correct}$ gain of approximately 1 dB. The superior performance can be attributed to the fact that the Stable Diffusion model is optimized for the $512 \times 512$ latent-space resolution. Processing the full latent $\mathbf{z}_{sec}(i)$ thus better exploits the model's inherent contextual understanding and global coherence, leading to significantly improved reconstruction quality.

\setlength{\tabcolsep}{3.5pt}
\begin{table}[ht]
\caption{Performance analysis across different diffusion strategies. {PSNR$_{\text{Correct}}$ denote the PSNR values between the original secret images and their reconstructions using the correct $\mathcal{K}_{priv}$.}}
\label{tab:diff_rule}
\begin{center}
\begin{small}
\begin{sc}
\begin{tabular}{cccc}
\toprule
\normalfont Method & \normalfont MANIQA$\uparrow$ & \normalfont CLIP Score$\uparrow$ & \normalfont PSNR$_\text{Correct}$$\uparrow$\\
\midrule
\normalfont separate denoising & 0.434 & 30.129 & 22.928\\
\normalfont joint noising & 0.417 & 30.074 & 23.901 \\
\normalfont MIDAS(Ours) & 0.434 & 30.129 & 23.903 \\
\bottomrule
\end{tabular}
\end{sc}
\end{small}
\end{center}
\vskip -0.1in
\end{table}

In contrast, applying joint noising during the Hiding Stage does not yield any improvement in either stego image quality or reconstructed image quality compared to separate processing.
 Moreover, adopting this complex joint process introduces a potential security risk by allowing information to mingle across encrypted latent segments associated with different users. Therefore, joint noising is intentionally omitted during the Hiding Stage.

\subsection{Effectiveness of Latent Vector Fusion.}
\label{sec:eff_LVF}
Since the original parameters of CRoSS* and DiffStega* were optimized for the $N=1$ case, the performance gap between these baselines and MIDAS may partly stem from parameter differences. To eliminate this factor, we construct DiffStega** by adopting the fair key strengths and diffusion steps used in MIDAS. We further evaluate MIDAS$_{\text{NF}}$, a variant of MIDAS where the Random Basis mechanism is replaced with a Noise Flip, to verify its specific effectiveness. Of note, $\text{MIDAS}_{\text{NF}}$ serves as an intermediate variant to bridge the gap between DiffStega** and MIDAS by incorporating a weighted interpolation with $\mathbf{z}_{ref}$ into the DiffStega** framework. This setup allows us to disentangle the performance gains provided by the fusion architecture from those achieved by the Random Basis mechanism.

Table \ref{tab:LVF} presents the results of our ablation study. While optimizing the parameters of DiffStega* (yielding DiffStega**) markedly improves diversity (higher CLIP Score) and reconstruction quality, it leads to a substantial degradation in perceptual stego image quality, as evidenced by a sharp drop in MANIQA scores. The introduction of our Latent Vector Fusion architecture effectively mitigates this degradation, significantly restoring stego image quality in MIDAS$_{\text{NF}}$. Notably, this architectural improvement not only recovers visual fidelity but also further enhances diversity, while maintaining robust reconstruction performance. Finally, the incorporation of the Random Basis mechanism provides further refinement, achieving the highest visual fidelity and diversity among all variants. These results demonstrate that the Latent Vector Fusion architecture and its Random Basis mechanism are essential for achieving the superior performance of MIDAS.
\setlength{\tabcolsep}{2pt}
\begin{table}[ht]
\caption{Effectiveness of Latent Vector Fusion. PSNR$_{\text{Correct}}$ and PSNR$_{\text{Wrong}}$ denote the PSNR values between the original secret images and their reconstructions using the correct and incorrect $\mathcal{K}_{priv}$, respectively.}
\label{tab:LVF}
\begin{center}
\begin{small}
\begin{sc}
\begin{tabular}{ccccc}
\toprule
\normalfont Method & \normalfont MANIQA$\uparrow$ & \normalfont CLIP Score$\uparrow$ & \normalfont PSNR$_\text{Correct}$$\uparrow$ & \normalfont PSNR$_\text{Wrong}$$\downarrow$\\
\midrule
\normalfont DiffStega* & 0.399 & 26.952 & 21.908 & 18.137 \\
\normalfont DiffStega** & 0.324 & 28.459 & 24.091 & 10.769 \\
\normalfont MIDAS$_\text{NF}$ & 0.380 & 29.631 & 23.616 & 10.566 \\
\normalfont MIDAS(Ours) & 0.434 & 30.129 & 23.903 & 9.964 \\
\bottomrule
\end{tabular}
\end{sc}
\end{small}
\end{center}
\vskip -0.1in
\end{table}

\subsection{Effectiveness of $\alpha$.}
\label{appendix:alpha}

Fig. \ref{abl:alpha} and Table \ref{tab:alpha} show the experimental results of MIDAS with varying $\alpha$, the parameter in Equation \eqref{eq:LVF} that balances $M_{D}\mathbf{z}_{prot}$ and $\mathbf{z}_{ref}$. As $\alpha$ decreases (i.e., the influence of the reference image grows), stego image quality is improved, however, this comes at the cost of reduced reconstruction quality. This trade-off is observed in Fig. 8, and quantitative results show that stego image quality quickly saturates, while reconstruction quality steadily degrades as $\alpha$ decreases.
\begin{figure}[ht]
\begin{center}
\center{\includegraphics[width=3in]{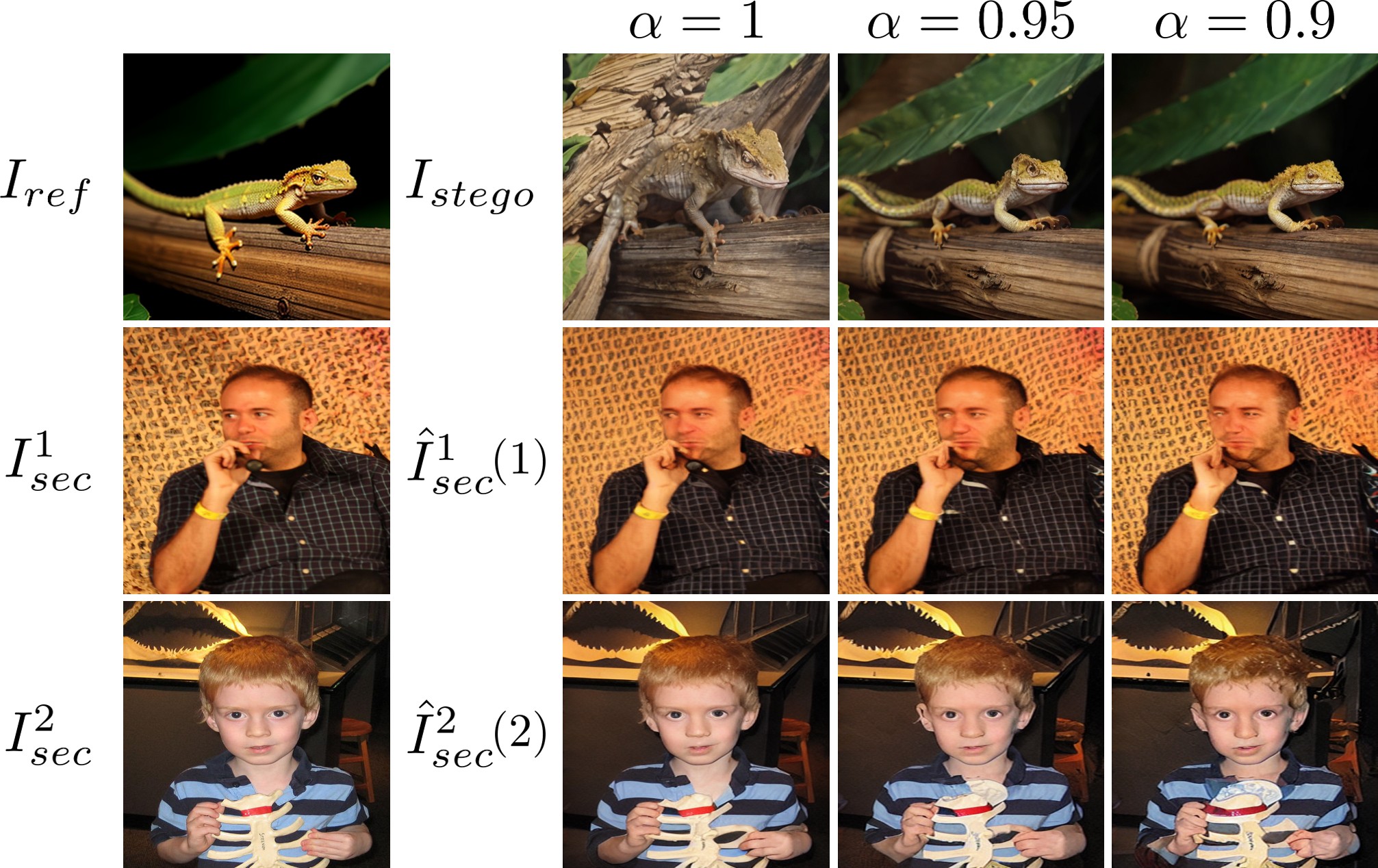}}
\caption{Experimental results of MIDAS with varying $\alpha$.}
\label{abl:alpha}
\end{center}
\end{figure}

\setlength{\tabcolsep}{2pt}
\begin{table}[ht]
\caption{Performance analysis of MIDAS with varying $\alpha$.}
\label{tab:alpha}
\begin{center}
\begin{small}
\begin{sc}
\begin{tabular}{ccccc}
\toprule
\normalfont $\alpha$ & \normalfont MANIQA$\uparrow$ & \normalfont CLIP Score$\uparrow$ & \normalfont PSNR$_\text{Correct}$$\uparrow$ & \normalfont PSNR$_\text{Wrong}$$\downarrow$\\
\midrule
\normalfont 0.9 & 0.436 & 30.212 & 23.286 &	9.755 \\
\normalfont 0.95 & 0.434 & 30.129 & 23.903 & 9.964 \\
\normalfont 1 & 0.394 & 29.044 & 24.481 & 10.271 \\
\bottomrule
\end{tabular}
\end{sc}
\end{small}
\end{center}
\vskip -0.1in
\end{table}

\newpage
\subsection{Effectiveness of $\gamma_{priv}$ and $\gamma_{fuse}$.}
\label{appendix:gamma}

Fig. \ref{abl:gamma} illustrates the stego images generated by MIDAS with varying values of $\gamma_{priv}$ and $\gamma_{fuse}$. Reconstructed images are omitted as their quality shows no significant change across these parameters. As shown in the figure, low values of $\gamma_{priv}$ and $\gamma_{fuse}$ result in poor stego quality, as the structural information of the secret images remains visible.  However, as these parameters increase, the structural interference diminishes, yielding visually natural and high-quality stego images. Tables \ref{tab:gamma_priv} and \ref{tab:gamma_fuse} demonstrate these effects. Notably, $\gamma_{priv}$ and $\gamma_{fuse}$ have little impact on reconstruction quality, while wrong-key reconstruction is inversely proportional to these parameters.

\begin{figure}[ht]
\begin{center}
\center{\includegraphics[width=4in]{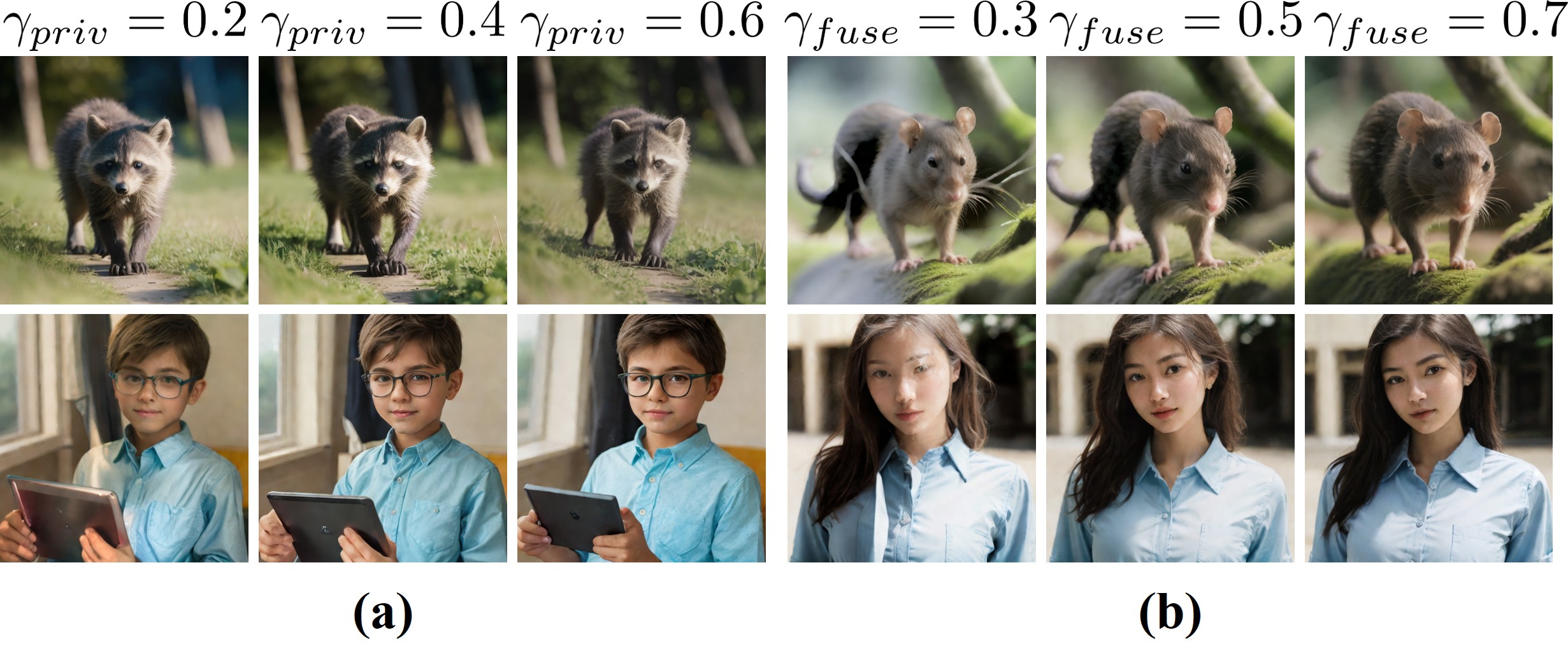}}
\caption{Visual comparisons of stego images generated by MIDAS with varying (a) $\gamma_{priv}$ and (b) $\gamma_{fuse}$.}
\label{abl:gamma}
\end{center}
\end{figure}

\setlength{\tabcolsep}{2pt}
\begin{table}[ht]
\caption{Performance analysis of MIDAS with varying $\gamma_{priv}$.}
\label{tab:gamma_priv}
\begin{center}
\begin{small}
\begin{sc}
\begin{tabular}{ccccc}
\toprule
\normalfont $\gamma_{priv}$ & \normalfont MANIQA$\uparrow$ & \normalfont CLIP Score$\uparrow$ & \normalfont PSNR$_\text{Correct}$$\uparrow$ & \normalfont PSNR$_\text{Wrong}$$\downarrow$\\
\midrule
\normalfont 0.2 & 0.406 & 29.923 & 23.840 &	13.875 \\
\normalfont 0.4 & 0.434 & 30.129 & 23.903 & 9.964 \\
\normalfont 0.6 & 0.451 & 30.181 & 23.921 & 8.153 \\
\bottomrule
\end{tabular}
\end{sc}
\end{small}
\end{center}
\vskip -0.1in
\end{table}

\setlength{\tabcolsep}{2pt}
\begin{table}[!h]
\caption{Performance analysis of MIDAS with varying $\gamma_{fuse}$.}
\label{tab:gamma_fuse}
\begin{center}
\begin{small}
\begin{sc}
\begin{tabular}{ccccc}
\toprule
\normalfont $\gamma_{fuse}$ & \normalfont MANIQA$\uparrow$ & \normalfont CLIP Score$\uparrow$ & \normalfont PSNR$_\text{Correct}$$\uparrow$ & \normalfont PSNR$_\text{Wrong}$$\downarrow$\\
\midrule
\normalfont 0.3 & 0.392 & 29.861 & 23.899 &	10.040 \\
\normalfont 0.5 & 0.434 & 30.129 & 23.903 & 9.964 \\
\normalfont 0.7 & 0.461 & 30.198 & 23.888 & 9.910 \\
\bottomrule
\end{tabular}
\end{sc}
\end{small}
\end{center}
\vskip -0.1in
\end{table}
\clearpage
\newpage

\subsection{Effectiveness of {C}heckpoints.}
\label{appendix:check}

Fig. \ref{abl:check} illustrates the experimental results of MIDAS with an alternative checkpoint while maintaining the same architecture.\footnote{{https://huggingface.co/stablediffusionapi/majicmix-fantasy}\label{fn:checkpoint}} As shown in the figure, MIDAS consistently produces high-quality stego images and successful reconstructions across different checkpoints, demonstrating its strong generalizability.

\begin{figure}[ht]
\begin{center}
\center{\includegraphics[width=3in]{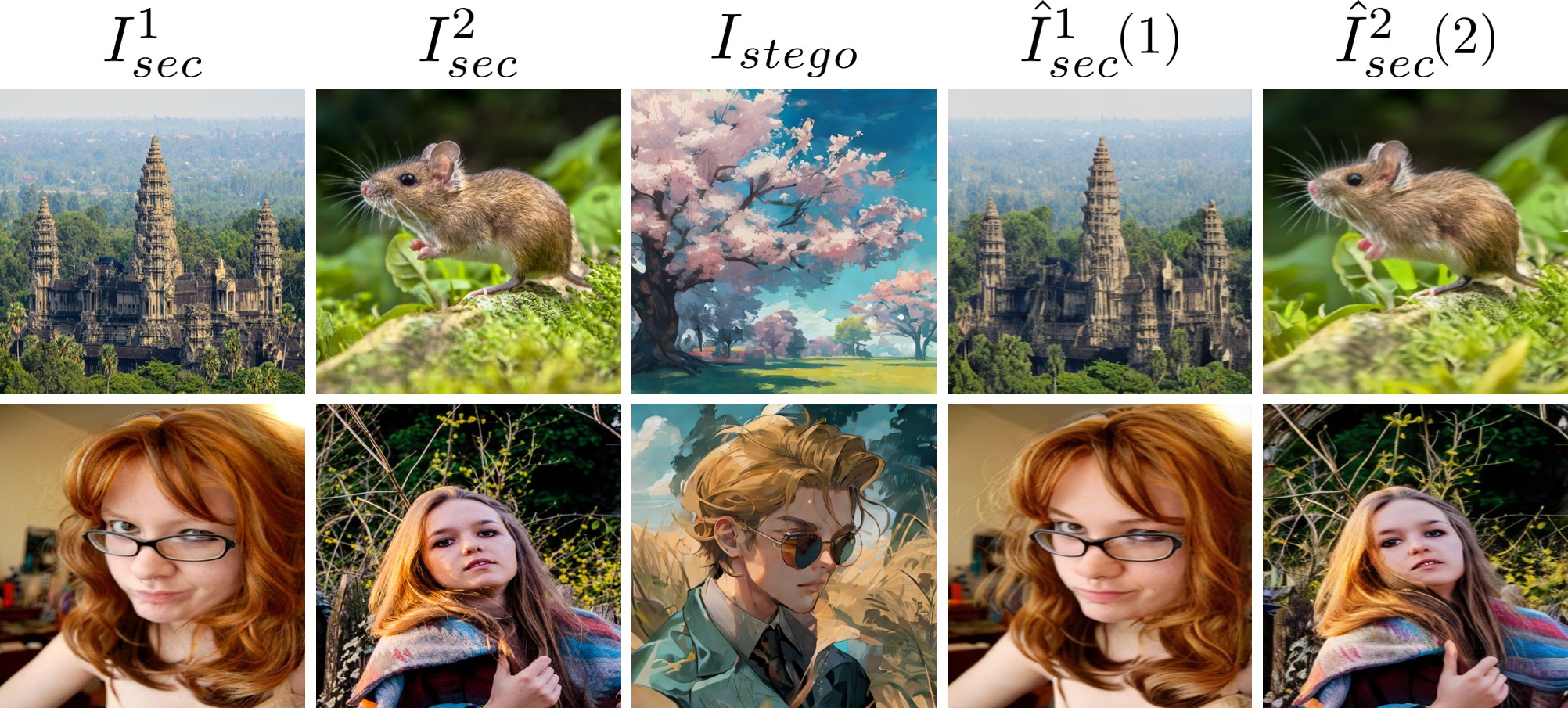}}
\caption{Experimental results of MIDAS using an alternative checkpoint\textsuperscript{\ref{fn:checkpoint}}.}
\label{abl:check}
\end{center}
\end{figure}

\newpage

\section{Proof of Theorem \ref{thm:random_basis}}
\label{appendix:proof}
Before proving Theorem \ref{thm:random_basis}, we first introduce several preliminary results.
\begin{lemma}[\citet{renyi1959}]
\label{lemma:renyi}
    Let $\mathbf{x}$ be a random vector with information dimension $d(\mathbf{x})$ and let $\mathbf{x}^\Delta$ denote the quantized version of $\mathbf{x}$ with quantization step $\Delta$.  
    If the $d(\mathbf{x})$-dimensional entropy $\mathbb H_{d(\mathbf{x})}(\mathbf{x})$ exists,  the entropy of $\mathbf{x}^\Delta$ (with $H(\cdot)$ denoting discrete entropy) is approximated as 
    \begin{equation}
        H(\mathbf{x}^\Delta) \approx \mathbb H_{d(\mathbf{x})}(\mathbf{x}) - d(\mathbf{x}) \log \Delta
    \end{equation}
for sufficiently small $\Delta$,  which follows from the definition of information dimension.
\end{lemma}

\begin{lemma}
\label{lemma:sphere}
    Let $\mathbf{x} \in \mathbb{R}^{n}$ be uniformly distributed on the surface of the $n$-dimensional sphere of radius $r$, with respect to the surface measure. Then
    \begin{equation}
        \mathbb H_{ d(\mathbf{x}) }(\mathbf{x}) 
        = \log\left( \frac{2\pi^{\frac{n}{2}}}{\Gamma\left(\frac{n}{2}\right)} r^{n-1} \right), 
        \quad d(\mathbf{x}) = n-1.
    \end{equation}
    This follows since $\mathbf{x}$ is uniformly distributed with respect to the $(n-1)$-dimensional surface measure, whose total mass equals the surface area of the sphere.
\end{lemma}

\begin{lemma}
\label{lemma:chi}
    Let $r\sim \chi_n$, where $\chi_n$ denotes the chi distribution with $n$ degrees of freedom. Then,
    \begin{equation}
        \mathbb H_{d(r)}(r) =\log \left(\frac{\Gamma\left(\frac n 2\right)}{\sqrt 2}\right)-\frac{n-1}{2}\psi\left(\frac {{n}} 2\right)+\frac n 2, \quad d(r) = 1.
    \end{equation}
    where $\psi(\cdot)$ denotes the digamma function. Moreover, for large $n$, 
    \begin{equation}
        \mathbb H_{d(r)}(r) \approx O(\log n).
    \end{equation}
\end{lemma}

\begin{lemma}[\citet{entropy}]
\label{lemma:sum_entropy}
    Let $\mathbf x$ and $\mathbf y$ be independent random variables with finite dimensional entropies $\mathbb H_{d(\mathbf x)}(\mathbf x)$ and $\mathbb H_{d(\mathbf y)}(\mathbf y)$, respectively. Then,
    \begin{equation}
        \mathbb H_{d(\mathbf x, \mathbf y)}(\mathbf x, \mathbf y)=\mathbb H_{d(\mathbf x)}(\mathbf x) + \mathbb H_{d(\mathbf y)}(\mathbf y),\quad d(\mathbf x, \mathbf y) = d(\mathbf x) + d(\mathbf y).
    \end{equation}
\end{lemma}



\begin{lemma}[\citet{Vershynin_2018}]
\label{lemma:polar}
    Let $\mathbf{x} \sim \mathcal{N}(0, I_{{n}})$ and represent it in polar form as
    \begin{equation}
        \mathbf{x} = r \boldsymbol{\theta},
    \end{equation}
    where $r = \|\mathbf{x}\|_2$ and $\boldsymbol{\theta} = \mathbf{x} / \|\mathbf{x}\|_2$.
    Then:
    \begin{enumerate}
        \item $r$ and $\boldsymbol{\theta}$ are independent.
        \item $\boldsymbol{\theta}$ is uniformly distributed on the surface of unit sphere $S^{{{n}}-1}$.
    \end{enumerate}
   Moreover, $r \sim \chi_{n}$.
\end{lemma}

The following proposition quantifies the information leakage after applying the Random Basis mechanism.

\begin{proposition}
\label{prop:mi}
    Let $\mathbf{z} \sim \mathcal{N}(0, I_d)$ and let $Q$ be an independent random orthonormal matrix distributed according to the Haar measure on the orthogonal group $O(d)$. Then, for sufficiently large $d$ and sufficiently small $\Delta$,
    \begin{equation}
        I\big(\mathbf{z}^\Delta; (Q\mathbf{z})^\Delta\big) \approx O(-\log {\Delta} + \log d)
    \end{equation}
\end{proposition}
\begin{proof}
   Let $r=\|\mathbf z\|_2, \;\mathbf z=r\boldsymbol\theta$, and $Q\mathbf z = r\boldsymbol\theta_Q$. Since $Q$ is orthonormal and independent of $\mathbf z\sim\mathcal N(0,I_d)$, we have $Q\mathbf z\sim\mathcal N(0,I_d)$ and $\|Q\mathbf z\|_2=\|\mathbf z\|_2=r$. Moreover, by Lemma~\ref{lemma:polar} and since $Q$ is uniformly random over $O(d)$, $r\sim\chi_d$, both $\boldsymbol\theta$ and $\boldsymbol\theta_Q$ are uniformly distributed on the surface of unit sphere $S^{d-1}$, and $r$, $\boldsymbol\theta$, and $\boldsymbol\theta_Q$ are mutually independent.
    \begin{align}
        I&\left(\mathbf z^\Delta;(Q\mathbf z)^\Delta\right) =H (\mathbf z^\Delta) + H((Q\mathbf z)^\Delta) - H(\mathbf z^\Delta, (Q\mathbf z)^\Delta)\\
        &\overset{(a)}{\approx} \mathbb H_{d(\mathbf z)}(\mathbf z)-d(\mathbf z)\log\Delta 
        +\mathbb H_{d(Q\mathbf z)}(Q\mathbf z)-d(Q\mathbf z)\log\Delta 
        -(\mathbb H_{d(\mathbf z, Q\mathbf z)}(\mathbf z, Q\mathbf z)-d(\mathbf z, Q\mathbf z)\log\Delta)\\
        &\overset{(b)}= \mathbb H_{d(r, \boldsymbol\theta)}(r, \boldsymbol\theta)-d(r, \boldsymbol\theta)\log\Delta 
        +\mathbb H_{d(r, \boldsymbol{\theta}_Q)}(r, \boldsymbol\theta_Q)-d(r, \boldsymbol\theta_Q)\log\Delta 
        -(\mathbb H_{d(r, \boldsymbol\theta, \boldsymbol\theta_Q)}(r, \boldsymbol\theta, \boldsymbol\theta_Q)-d(r, \boldsymbol\theta, \boldsymbol\theta_Q)\log\Delta)\\
        &\overset{(c)}{=} \mathbb H_{d(r)}(r)-d(r)\log\Delta\\
        &\overset{(d)}{\approx} O(-\log \Delta + \log d)
    \end{align} where (a) follows from Lemma~\ref{lemma:renyi}, (b) follows from the polar representations, (c) follows from Lemma~\ref{lemma:sum_entropy}, the mutual independence of $r$, $\boldsymbol\theta$, and $\boldsymbol\theta_Q$, together with Lemmas~\ref{lemma:sphere}, and \ref{lemma:chi}, and (d) follows from Lemma~\ref{lemma:chi}.
\end{proof}

Building on Proposition \ref{prop:mi}, we now prove the Theorem \ref{thm:random_basis}.
\begin{proof}
    For simplicity, we consider the following abstraction of our pipeline:
    \begin{equation}
        I_{sec}\rightarrow \mathbf z_{sec}^\Delta=((\mathbf z_{sec}^{\gamma})^\Delta,(\mathbf z_{sec}^{1-\gamma})^\Delta)\xrightarrow{\text{Random Basis}}\mathbf z_{prot}^\Delta=((Q\mathbf z_{sec}^{\gamma})^\Delta,(\mathbf z_{sec}^{1-\gamma})^\Delta)\rightarrow\hat{I}_{sec}
    \end{equation}
    Here, $\mathbf z_{sec}^\Delta$ denotes the quantized version of the noisy latent representation $\mathbf z_{sec}\sim \mathcal{N}(0,I_{d})$ of the secret image $I_{sec}$. Note that $d<m$, since the latent representation is a downsampled version of the image.  $\mathbf z_{prot}^\Delta$ is the quantized version of the protected latent $\mathbf z_{prot}$. Specifically, $\mathbf z_{sec}^\Delta$ is partitioned into $(\mathbf z_{sec}^{\gamma})^\Delta$ ($\gamma$ fraction of $\mathbf z_{sec}$) and $(\mathbf z_{sec}^{1-\gamma})^\Delta$ ($(1-\gamma)$ fraction of $\mathbf z_{sec}$), where only the former is rotated by a random orthonormal matrix $Q$ (private key) by the Random Basis mechanism. The protected latent $\mathbf z_{prot}^\Delta$ is then observed by the attacker to produce $\hat I_{sec}$.
    
    The mutual information $I(I_{sec};\hat I_{sec})$ is bounded as
    \begin{equation}
        I(I_{sec};\hat I_{sec})\leq I({\mathbf z}_{sec}^\Delta;{\mathbf z}_{prot}^\Delta)
    \end{equation}
    from the data processing inequality. This provides a formal basis for analyzing security in the latent space. By independence, we have 
    \begin{equation}
        I({\mathbf z}_{sec}^\Delta;{\mathbf z}_{prot}^\Delta)
        = I((Q\mathbf z_{sec}^{\gamma})^\Delta;(\mathbf z_{sec}^{\gamma})^\Delta)
        + I((\mathbf z_{sec}^{1-\gamma})^\Delta;(\mathbf z_{sec}^{1-\gamma})^\Delta).
    \end{equation}
    Then, by Proposition \ref{prop:mi} and $d<m$, the first term satisfies
    \begin{equation}
        I((Q\mathbf z_{sec}^{\gamma})^\Delta;(\mathbf z_{sec}^{\gamma})^\Delta)
        \approx  O(-\log {\Delta} + \log {m}).
    \end{equation}
    The second term is
    \begin{align}
        I((\mathbf z_{sec}^{1-\gamma})^\Delta;(\mathbf z_{sec}^{1-\gamma})^\Delta)
        &= H((\mathbf z_{sec}^{1-\gamma})^\Delta) \\
        &{\overset{(a)}{=}} \mathbb H_{(1-\gamma){d}}(\mathbf z_{sec}^{1-\gamma})
        - {d}(\mathbf z_{sec}^{1-\gamma})\log \Delta \\
        &{\overset{(b)}{\approx}}  O(-(1-\gamma)m \log {\Delta} + (1-\gamma)m).
    \end{align}
    where $(a)$ follows from Lemma~\ref{lemma:renyi} and $(b)$ follows from the fact that $\mathbf z_{sec}^{1-\gamma}\sim \mathcal{N}(0, I_{(1-\gamma){d}})$ and $d<m$.
    
    Therefore,
      \begin{equation}
        R_L = \frac 1 {{m}}I(I_{sec};\hat I_{sec})\leq \frac 1 {{m}}I({\mathbf z}_{sec}^\Delta;{\mathbf z}_{prot}^\Delta)
        \approx O\left( \frac{-\log \Delta+ \log {m}}{{m}}+(1-\gamma)(-\log\Delta+1)\right).
    \end{equation}
\end{proof}

\end{document}